\documentclass{article}

\usepackage[preprint]{neurips_2026}

\usepackage[utf8]{inputenc}
\usepackage[T1]{fontenc}
\usepackage{hyperref}
\usepackage{url}
\usepackage{booktabs}
\usepackage{array}
\usepackage{graphicx}
\usepackage{algorithm}
\usepackage{algpseudocode}
\usepackage{algorithmicx}
\usepackage{amsmath,amssymb,amsfonts,amsthm,mathtools}
\usepackage{nicefrac}
\usepackage{microtype}
\usepackage{xcolor}
\usepackage{bm}
\usepackage{aliascnt}
\usepackage[capitalize,noabbrev]{cleveref}
\usepackage{comment}

\theoremstyle{plain}
\newtheorem{theorem}{Theorem}[section]
\newaliascnt{lemma}{theorem}
\newtheorem{lemma}[lemma]{Lemma}
\aliascntresetthe{lemma}
\newaliascnt{corollary}{theorem}

\aliascntresetthe{corollary}
\newaliascnt{proposition}{theorem}

\aliascntresetthe{proposition}
\newaliascnt{claim}{theorem}

\aliascntresetthe{claim}

\theoremstyle{definition}
\newaliascnt{definition}{theorem}
\newtheorem{definition}[definition]{Definition}
\aliascntresetthe{definition}
\newaliascnt{assumption}{theorem}

\aliascntresetthe{assumption}
\newaliascnt{example}{theorem}

\aliascntresetthe{example}
\newaliascnt{question}{theorem}

\aliascntresetthe{question}

\newaliascnt{obs}{theorem}
\newtheorem{obs}[obs]{Observation}
\aliascntresetthe{obs}

\newaliascnt{remark}{theorem}
\newtheorem{remark}[remark]{Remark}
\aliascntresetthe{remark}

\crefname{theorem}{Theorem}{Theorems}
\Crefname{theorem}{Theorem}{Theorems}
\crefname{lemma}{Lemma}{Lemmas}
\Crefname{lemma}{Lemma}{Lemmas}
\crefname{corollary}{Corollary}{Corollaries}
\Crefname{corollary}{Corollary}{Corollaries}
\crefname{proposition}{Proposition}{Propositions}
\Crefname{proposition}{Proposition}{Propositions}
\crefname{claim}{Claim}{Claims}
\Crefname{claim}{Claim}{Claims}
\crefname{definition}{Definition}{Definitions}
\Crefname{definition}{Definition}{Definitions}
\crefname{assumption}{Assumption}{Assumptions}
\Crefname{assumption}{Assumption}{Assumptions}
\crefname{example}{Example}{Examples}
\Crefname{example}{Example}{Examples}
\crefname{question}{Question}{Questions}
\Crefname{question}{Question}{Questions}
\crefname{remark}{Remark}{Remarks}
\Crefname{remark}{Remark}{Remarks}
\crefname{appendix}{Appendix}{Appendices}
\Crefname{appendix}{Appendix}{Appendices}

\newcommand{\N}{\mathbb{N}}
\newcommand{\E}{\mathbb{E}}
\newcommand{\Prob}{\mathbb{P}}
\newcommand{\R}{\mathbb{R}}

\newcommand{\Z}{\mathbb{Z}}

\newcommand{\G}{\mathcal{G}}

\newcommand{\U}{\mathcal{U}}
\newcommand{\V}{\mathcal{V}}
\newcommand{\W}{\mathcal{W}}

\newcommand{\x}{\mathbf{x}}
\newcommand{\y}{\mathbf{y}}
\newcommand{\z}{\mathbf{z}}
\newcommand{\w}{\mathbf{w}}
\renewcommand{\a}{\mathbf{a}}
\renewcommand{\b}{\mathbf{b}}

\newcommand{\q}{\mathbf{q}}

\renewcommand{\u}{\mathbf{u}}
\renewcommand{\v}{\mathbf{v}}
\newcommand{\0}{\mathbf{0}}
\newcommand{\1}{\mathbf{1}}

\newcommand{\old}[1]{{}}

\DeclareMathOperator{\argmin}{arg\,min}

\title{Sample Complexity of Stochastic Optimization with Integer Variables}

\author{%
Hongyu Cheng\\
Dept. of Applied Mathematics \& Statistics\\
Johns Hopkins University\\
Baltimore, MD 21218\\
\texttt{hongyucheng@jhu.edu}
\And
Yinghao Zheng\\
Dept. of Applied Mathematics \& Statistics\\
Johns Hopkins University\\
Baltimore, MD 21218\\
\texttt{yzheng80@jhu.edu}
\AND
Marco Molinaro\\
Microsoft Research (Redmond)\\
Dept. of Computer Science, PUC-Rio\\
\texttt{mmolinaro@microsoft.com}
\And
Amitabh Basu\\
Dept. of Applied Mathematics \& Statistics\\
Johns Hopkins University\\
Baltimore, MD 21218\\
\texttt{basu.amitabh@jhu.edu}
}

\begin{document}

\maketitle

\begin{abstract}
We establish sample complexity results for stochastic optimization over the integers, especially with a view to understand the complexity with respect to the corresponding continuous optimization problem. We show that integer optimization can sometimes require strictly more samples and sometimes strictly smaller number of samples, depending on the structure of the objective and constraints. 
\begin{enumerate}
    \item For Lipschitz objectives over subsets of the $\ell_\infty$ ball, the statistical complexity of general stochastic mixed-integer, nonlinear, nonconvex optimization is exactly the same as stochastic linear optimization with just bound constraints.
    \item For Lipschitz objectives over subsets of the $\ell_2$ ball, we show that integer optimization can require strictly {\em smaller} sample size compared to the continuous setting in a certain regime. To get to this result, we also establish tight sample complexity results for nonconvex continuous stochastic optimization which, to the best of our knowledge, do not appear in prior work.
    \item For strongly convex, smooth objectives, integer optimization has high statistical complexity compared to the continuous setting. In particular, we show that integer optimization requires $\Omega(1/\epsilon^2)$ samples to report an $\epsilon$-approximate solution, compared to the well-known $O(1/\epsilon)$ sample complexity from the continuous optimization literature.
\end{enumerate}
\end{abstract}

\section{Introduction}

The fundamental question in statistical machine learning is a stochastic optimization problem:

\begin{equation}\label{eq:SCO}\min_{\x \in X} F_{D}(\x) := \mathbb{E}_{\z \sim D}[f(\x; z)]\end{equation} where $X \subseteq \R^d$ denotes the constraints the decision variable $\x$ must satisfy, $f : X \times \mathcal{Z} \to \R$ is a fixed function, $\mathcal{Z}$ is a measurable space and $D$ is a probability distribution on $\mathcal{Z}$. The goal is to solve this problem (to within some desired error tolerance) without having explicit knowledge of the distribution $D$ and only having access to i.i.d samples from $D$. We must ``learn'' enough about the objective function $F_D$, which depends on the distribution $D$, from the sample, to be able to report a good solution with high confidence, i.e., high probability over the random sample~\cite{shalev2014understanding}. In this paper, we investigate the central notion of sample complexity, i.e., the smallest number of i.i.d samples from $D$ that are needed to solve the problem to a given level of accuracy with high confidence. This obviously depends on the nature of the function $f$ (e.g., assuming $f(\cdot\,; z)$ is convex for every $z \in \mathcal{Z}$), the constraint set $X$, as well as the family of distributions $D$ one wants to target. The focus is typically on the structure of $f$ and $X$, with minimal assumptions on $D$.

While the sample complexity question has been studied extensively in the continuous convex optimization setting (we summarize the state-of-the-art results below), the setting with integer decision variables seems to have received much less attention, to the best of our knowledge. This paper is an attempt to close this gap and in the process understand the relative difficulty of continuous versus discrete stochastic optimization in terms of quantitative sample complexity bounds. It has been known for a while that the algorithmic and information complexity of convex optimization with integer variables is much higher (exponential in the dimension) than continuous convex optimization~\cite{SchrijverBOOK,conforti2014integer,basu2025convexity}. To the best of our knowledge, the statistical aspect has not been rigorously studied: Does stochastic optimization with integer variables need a much larger number of samples compared to its continuous counterpart? This paper makes progress on this statistical question.

\begin{table}[htbp]
\centering
\caption{Summary of sample complexity bounds.}\label{table:results}
\label{tab:discrete}
\begingroup
\scriptsize
\setlength{\tabcolsep}{2.2pt}
\renewcommand{\arraystretch}{1.45}
\newcommand{\ratecell}[1]{%
  \begingroup
  \setbox0=\hbox{$#1$}%
  \ifdim\wd0>\linewidth
    \makebox[\linewidth][c]{\resizebox{\linewidth}{!}{\box0}}%
  \else
    \makebox[\linewidth][c]{\box0}%
  \fi
  \endgroup}
\begin{tabular}{@{}>{\raggedright\arraybackslash}m{0.15\linewidth}>{\centering\arraybackslash}m{0.28\linewidth}>{\centering\arraybackslash}m{0.28\linewidth}>{\centering\arraybackslash}m{0.25\linewidth}@{}}
\toprule
Problem class & Any data-driven algorithm & ERM & UC \\
\midrule
\begin{tabular}[c]{@{}l@{}}
$X \subseteq B_R^{(\infty)}$,  \\
$1$-Lipschitz $f$
\end{tabular}
&
\ratecell{\Theta\!\left(\tfrac{R^2}{\epsilon^2}(d+\log(1/\delta))\right)}
&
\ratecell{\Theta\!\left(\tfrac{R^2}{\epsilon^2}(d+\log(1/\delta))\right)}
&
\parbox[c]{\linewidth}{\centering
\ratecell{\Theta\!\left(\tfrac{R^2}{\epsilon^2}(d+\log(1/\delta))\right)}%
}
\\
\addlinespace[0.55em]
\begin{tabular}[c]{@{}l@{}}
$X \subseteq B_R^{(2)}$,  \\
$1$-Lipschitz $f$
\end{tabular}
&
\ratecell{\Theta\!\left(\tfrac{R^2}{\epsilon^2}(d+\log(1/\delta))\right)}
&
\ratecell{\Theta\!\left(\tfrac{R^2}{\epsilon^2}(d+\log(1/\delta))\right)}
&
\parbox[c]{\linewidth}{\centering
\ratecell{\Theta\!\left(\tfrac{R^2}{\epsilon^2}(d+\log(1/\delta))\right)}%
}
\\
\addlinespace[0.55em]
\begin{tabular}[c]{@{}l@{}}
$X = B_R^{(2)} \cap \Z^d$, \\
$1$-Lipschitz $f$
\end{tabular}
&
\ratecell{\Theta\!\left(\tfrac{R^2}{\epsilon^2}(H_2(d,R)+\log(1/\delta))\right)}
&
\ratecell{\Theta\!\left(\tfrac{R^2}{\epsilon^2}(H_2(d,R)+\log(1/\delta))\right)}
&
\ratecell{\Theta\!\left(\tfrac{R^2}{\epsilon^2}(H_2(d,R)+\log(1/\delta))\right)}%
\\
\addlinespace[1.28em]
\begin{tabular}[c]{@{}l@{}}
$X=B_R^{(\infty)} \cap \Z^d$\\
$L$-smooth and\\
$\mu$-strongly convex $f$
\end{tabular}
&
\ratecell{\Theta\!\left(\tfrac{\sigma^2d\,\bar{\kappa}}{\epsilon^2}(d+\log(1/\delta))\right)}
&
\ratecell{\Theta\!\left(\tfrac{\sigma^2d\,\bar{\kappa}}{\epsilon^2}(d+\log(1/\delta))\right)}
&
\ratecell{\Theta\!\left(\tfrac{\sigma^2\lfloor R\rfloor^2 d}{\epsilon^2}(d+\log(1/\delta))\right)}
\\
\bottomrule
\end{tabular}
\vspace{0.35em}
\begin{minipage}{0.98\linewidth}
Here %
$H_2(d,R)=0$ for $0<R<1$, while $H_2(d,R)=\min\{d,\lfloor R^2\rfloor\}\log(ed/\min\{d,\lfloor R^2\rfloor\})$ for $R\ge1$, and $\bar{\kappa}:=\min\{L/\mu,\lfloor R\rfloor^2\}$. For the smooth strongly convex row, we allow $R=\infty$ with the interpretation that $B_\infty^{(\infty)}=\R^d$ and $\lfloor\infty\rfloor=\infty$. The uniform convergence bounds require an {\em anchoring} assumption stated formally in Section~\ref{sec:setup} (without such an assumption, the sample complexity of UC can be shown to be $+\infty$ in all the cases above).%
\end{minipage}
\endgroup
\end{table}

\subsection{Our contributions and relation to prior work}
We investigate three standard notions of sample complexity within stochastic optimization: for uniform convergence, for empirical risk minimization, and sample complexity for any decision rule/algorithm. {\em Uniform convergence (UC)} analysis aims to provide bounds on the minimum number of samples needed so that $|F_S(\x) - F_D(\x)| \leq \epsilon$ for all $\x \in X$, for a given tolerance $\epsilon > 0$ with probability $1 - \delta$ for some confidence parameter $\delta > 0$, where \begin{equation}\label{eq:ERM-SAA}
F_S(\x):=\frac{1}{m}\sum_{i=1}^m f(\x;z_i)
,\end{equation} is the so-called {\em empirical risk minimization (ERM)} or {\em sample average approximation (SAA)} objective based on the i.i.d. samples $z_1, \ldots, z_m \sim D$. If uniform convergence holds, then minimizing the ERM objective~\eqref{eq:ERM-SAA} to within $\epsilon/2$ error will guarantee a solution to~\eqref{eq:SCO} to within $\epsilon$ error with probability $1-\delta$. While uniform convergence is thus sufficient for ERM to work, it is not necessary in general; for example, one could use a smaller number of samples to approximate the objective $F_D$ in a small neighborhood around the minimizer (so that ERM can succeed), as opposed to over all of the domain $X$. This was first rigorously established in~\cite{shalev2009stochastic} for stochastic convex optimization and sharpened in~\cite{feldman2016generalization} and~\cite{carmon2024sample} (see the discussion in Appendix~\ref{sec:SCO-summary}). Thus, an important question within stochastic optimization is to decide whether ERM has strictly smaller sample complexity compared to UC. Finally, minimizing the ERM objective~\eqref{eq:ERM-SAA} is just one way to use the data $z_1, \ldots, z_m$. In general, one could consider any function $A_m(z_1, \ldots, z_m)$ that returns a solution $\hat x \in X$ for the problem~\eqref{eq:SCO}. Potentially, one could get away with an even smaller number of samples compared to ERM because one is allowed to do other kinds of computations on the data. In fact, this can happen for stochastic convex optimization as a result of lower bounds on ERM in~\cite{feldman2016generalization} and results on stochastic gradient descent and online optimization as discussed in~\cite{shalev2009stochastic} (see Appendix~\ref{sec:SCO-summary}). In Table~\ref{table:results}, we put the sample complexity for a general algorithm/solution estimator under the column ``Any data-driven algorithm''. As the above discussion shows, the sample complexity of UC is always at least as large as the sample complexity of ERM which is at least as large as the sample complexity when allowing for any solution estimator.

\paragraph{Characterization of Lipschitz stochastic optimization over $\ell_\infty$.} We first consider stochastic optimization where each loss function  $\x \mapsto f(\x\,;z)$ is 1-Lipschitz with respect to $\ell_\infty$ over the box of radius $R$, namely $B^{(\infty)}_R := \{\x \in \R^d: \|\x\|_\infty \leq R\}$. Here we give a \emph{complete characterization of the sample complexity} across the different modes of convergence: uniform convergence, ERM, and any data-driven algorithm.%
We show that the sample complexity is the same in all these settings, and equal to $\Theta\big(\tfrac{R^2}{\epsilon^2}(d+\log(\tfrac{1}{\delta}))\big)$. This is summarized in the first row of Table~\ref{table:results}, and formally stated in Section~\ref{sec:linfty-ball}. An important message of this set of results is that the most complicated mixed-integer nonlinear, nonconvex problems have the same sample complexity as simple stochastic linear programming problems with just bound constraints.

To establish this, we prove an upper bound on uniform convergence for Lipschitz functions (possibly nonconvex) over the whole of $B^{(\infty)}_R$, which implies the same upper bound for all the other settings; note it includes the particularly interesting case of \emph{mixed-integer convex} optimization, where $X$ is the set of mixed-integer points in some convex subset of $B^{(\infty)}_R$, i.e., $X = C \cap (\Z^{n} \times \R^{d-n})$ for some convex set $C$ and $n \leq d$. The characterization is then obtained by a matching lower bound for any learning algorithm for linear loss functions over the whole box $X = B^{(\infty)}_R$, i.e., it is achieved by  stochastic linear programming with simple variable bound constraints; adding integrality does not change anything, since the optimal solutions are attained at the vertices of the box, and so the lower bound holds even for $X = B^{(\infty)}_R \cap \Z^d$.

\paragraph{Equivalences and separations under $\ell_2$-norm.}
Next, we show how the geometry of the bounding region (scaled box, in the previous part) can affect the sample complexity of the different problems. More concretely, we now consider loss functions $\x \mapsto f(\x\,;z)$ is 1-Lipschitz with respect to $\ell_2$ over the \emph{Euclidean ball} of radius $R$, namely $B^{(2)}_R := \{\x \in \R^d: \|\x\|_2 \leq R\}$. As discussed above, when the loss functions are \emph{convex} and the feasible region is the continuous set $X = B^{(2)}_R$, existing results imply a separation between the sample complexity of uniform convergence, ERM, and a general algorithm~\cite{shalev2009stochastic,feldman2016generalization,carmon2024sample}; see Table~\ref{table:SCO} in Appendix~\ref{sec:SCO-summary} for the precise bounds.

This leads the question of what happens when one allows nonconvexity in the objective and integrality constraints: we saw above that they do not make any difference over the $\ell_\infty$ ball. First, we show that if we do not have integrality constraints, but allow nonconvex Lipschitz objectives over the $\ell_2$ ball, the sample complexity is the same as the $\ell_\infty$ case, namely  $\Theta\!\big(\tfrac{R^2}{\epsilon^2}(d+\log(\tfrac{1}{\delta}))\big)$ for UC, ERM, and general algorithms. This is summarized in the second row of Table~\ref{table:results} and stated in \Cref{thm:l2-continuous-nonconvex-upper-lower}. %
In particular, the separation between anchored UC, ERM and general proper learning algorithms for the convex case does not translate to the nonconvex setting, and there is no gain to be had using anything other than ERM. This is significant because many important problems in modern machine learning and data science are nonconvex stochastic optimization problems, e.g., neural network based approaches.

Allowing integrality constraints on top of nonconvexity in the $\ell_2$ ball makes the story even more intriguing. Consider minimizing over the integer points in the $\ell_2$ ball, i.e., $X = B_R^{(2)} \cap \Z^d$. We show that the sample complexity is now characterized by a term $H_2(d,R)$ arising out of the geometry of the integer points in the $\ell_2$ ball, where $H_2(d,R)=0$ for $0<R<1$ and $H_2(d,R)=\min\{d,\lfloor R^2\rfloor\}\log(ed/\min\{d,\lfloor R^2\rfloor\})$ for $R\ge1$. More precisely, the sample complexity is now $\Theta\big(\tfrac{R^2}{\epsilon^2}(H_2(d,R)+\log(\frac{1}{\delta}))\big)$, for all modes of convergence: UC, ERM, any algorithm. This is summarized in the third row of Table~\ref{table:results}. Comparing with the bound for the continuous feasible set $X = B^{(2)}_R$ we see that this is actually {\em strictly smaller than the continuous bound} when $R^2 < d$. That is, perhaps surprisingly, this stochastic integer-constrained problem can be statistically easier than the corresponding continuous version.
Intuitively, this happens because the convex hull of the integer points in the $\ell_2$ ball has a very different geometry compared to the ball when the radius is small (less than $\sqrt{d}$); which does not happen in the $\ell_\infty$ case. %

\old{

\paragraph{Lipschitz objective functions.} We first consider the problem where $f(\x\,;z)$ is 1-Lipschitz in $\x$ over a bounded domain $B\subseteq \R^d$ for every $z \in \mathcal{Z}$, and $X \subseteq B$. We allow $f(\x\,;z)$ to be nonconvex for some $z \in \mathcal{Z}$.

When $B$ is the $\ell_\infty$ ball of radius $R$, i.e., $B = B^{(\infty)}_R := \{\x \in \R^d: \|\x\|_\infty \leq R\}$ and $f(\x\,;z)$ is 1-Lipschitz in $\x$ with respect to the $\ell_\infty$ norm, we show that the uniform convergence sample complexity is upper bounded by $O\!\left(\tfrac{R^2}{\epsilon^2}(d+\log(\tfrac{1}{\delta}))\right)$, over all of $B^{(\infty)}_R$ and therefore for any subset $X \subseteq B^{(\infty)}_R$. This follows from fairly standard tools for bounding uniform convergence rates using covering number and chaining; for completeness, we include a proof in the appendix. We establish a matching lower bound of $\Omega\!\left(\tfrac{R^2}{\epsilon^2}(d+\log(\tfrac{1}{\delta}))\right)$ on the number of samples needed by any algorithm or solution estimator to solve~\eqref{eq:SCO} in this setting. This is summarized in the first row of Table~\ref{table:results} and the results are formally stated in Section~\ref{sec:linfty-ball}. As noted above, this settles the sample complexity of UC, ERM and any algorithm for Lipschitz functions over an arbitrary subset $X$ of the $\ell_\infty$ ball. Note, in particular, this includes the case where $X$ is the set of mixed-integer points in some convex subset of $B^{(\infty)}_R$, i.e., $X = C \cap (\Z^{n} \times \R^{d-n})$ for some convex set $C$ and $n \leq d$. Moreover, the lower bound is attained even in the setting when $X = B^{(\infty)}_R$, i.e., there are no further constraints beyond bound constraints on the variables, and the distribution $D$ is supported on $z \in \mathcal{Z}$ such that $f(\x, \z)$ is linear in $\x$. In other words, the lower bound comes from the sample complexity of stochastic linear programming with simple bound constraints. Adding integer constraints does not change anything, i.e., this lower bound continues to hold if $X = B^{(\infty)}_R \cap \Z^d$,  simply because minimizing linear functions over $X$ is the same as minimizing over the convex hull which is $B^{(\infty)}_R$. This shows, amongst other things, that the sample complexity of solving~\eqref{eq:SCO} with integrality constraints on some subset of $B^{(\infty)}_R$ is the same as that of linear programming with bound constraints. Therefore, within the $\ell_\infty$ ball, the sample complexity of nonconvex and convex stochastic optimization with Lipschitz objectives, with or without integer constrained decision variables is exactly the same as the simplest stochastic linear optimization problem! Note that for very simple feasible regions $X\subseteq B^{(\infty)}_R$, the sample complexity could be smaller; for example, if $X$ is a singleton, then we need zero samples. The point here is that even something as simple as (integer) linear programming with just bound constraints has the same sample complexity as the most complicated mixed-integer nonlinear, nonconvex problem.

This should be contrasted with the situation in continuous convex optimization over the $\ell_2$ ball $B^{(2)}_R:= \{\x\in\R^d:\ \|\x\|_2\le R\}$ when $f(\x\,;z)$ is 1-Lipschitz in $\x$ with respect to the $\ell_2$ norm. Here, the sample complexity of UC is $\Theta\!\left(\tfrac{R^2}{\epsilon^2}(d+\log(\tfrac{1}{\delta}))\right)$; the upper bound again comes from a similar covering numbers argument as the $\ell_\infty$ case, and the lower bound comes about by combining~\cite[Theorem 3.9]{feldman2016generalization} and the lower bound stated above for $B^{(\infty)}_R$ with $d=1$ (so $B^{(2)}_R = B^{(\infty)}_R$). The sample complexity of ERM is $\Theta\!\left(\tfrac{Rd}{\epsilon}+\tfrac{1}{\epsilon^2}\log(\tfrac{1}{\delta}))\right)$; the upper bound follows from Theorems 1 and 2 in~\cite{carmon2024sample} and the lower bound is from~\cite[Theorem 3.10]{feldman2016generalization}. Finally, one can use stochastic approximation (SA) methods or stochastic gradient descent (SGD) to obtain an $\epsilon$-approximate solution to~\eqref{eq:SCO} with probability $1-\delta$ with at most $O\left(\tfrac{1}{\epsilon^2}\log(\tfrac{1}{\delta})\right)$ samples; and no algorithm can do better because of the lower bound stated above for $B^{(\infty)}_R$ with $d=1$. Thus, over the $\ell_2$ ball, the sample complexity of uniform convergence for {\em convex}, Lipschitz functions is strictly larger than for ERM, which itself can be strictly improved by using SGD. 

This leads one to ask what happens in the $\ell_2$ ball if we allow nonconvexity in the objective and integrality constraints: we saw above that they do not make any difference over the $\ell_\infty$ ball. First, we show that if we do not have integrality constraints, but allow nonconvex Lipschitz objectives over the $\ell_2$ ball, the sample complexity is the same as the $\ell_\infty$ case. This is summarized in the second row of Table~\ref{table:results} and stated in \Cref{thm:l2-continuous-nonconvex-upper-lower}. %
In particular, the difference between UC, ERM and SCO for the convex case does not translate to the nonconvex setting, and there is no gain to be had using anything other than ERM, and in fact, one might as well get enough samples to approximate the true objective over the entire domain. This is significant because many important problems in modern machine learning and data science are nonconvex stochastic optimization problems, e.g., neural network based approaches.

Allowing integrality constraints on top of nonconvexity in the $\ell_2$ ball makes the story even more intriguing. Consider minimizing over the integer points in the $\ell_2$ ball, i.e., $X = B_R^{(2)} \cap \Z^d$ in~\eqref{eq:SCO}. We show that if $R^2 \geq d$, then the sample complexity for UC, ERM and any algorithm are all the same and equal to the continuous optimization sample complexity, i.e., $\Theta\!\left(\tfrac{R^2}{\epsilon^2}(d+\log(\tfrac{1}{\delta}))\right)$. On the other hand, if $R^2 < d$ then the sample complexity for UC, ERM and any algorithm are again all the same, but {\em strictly smaller than the continuous bound}. This is summarized in the third row of Table~\ref{table:results}, where the $H_2(d,R)$ term is strictly smaller than $d$ when $\lfloor R^2 \rfloor < d$, and is equal to $d$ otherwise.}

\paragraph{Strongly convex, smooth objectives.} We also investigate stochastic optimization of smooth and strongly convex functions. This is an important setting where, in continuous convex optimization, one can obtain sample-complexity  bounds with faster rates as a function of the error $\epsilon$~\cite[Chapter 8]{shapiro2021lectures}. As in prior works (see~\cite{harvey2019tight,harvey2019simple} and the references therein), we consider a compact feasible set $X$ and further assume the loss functions to  $1$-Lipschitz over $X$.

Our first main result in this setting is that stochastic integer optimization has a {\em much higher sample complexity} compared to continuous optimization. More specifically, we show that any proper learning algorithm requires at least $\Omega\left(\frac{1}{\epsilon^2}\right)$ samples for optimizing over the integer points in a box, as a function of $\epsilon$. In comparison, it is known that the sample-complexity of optimizing over a continuous box is $O\left(\frac{1}{\epsilon}\right)$~\cite{nemirovski2009robust,harvey2019tight} (see also \Cref{thm:sc-continuous-erm} in this paper). Thus, unlike in the previous $\ell_2$ case, imposing integrality constraints makes the problem harder. 

Moreover, we also consider stochastic smooth/strongly convex optimization over all integer points $\Z^d$. For that, we impose an assumption that the distribution over the loss functions satisfies a sub-Gaussian increment condition. This natural condition generalizes the previous  1-Lipschitz assumption to unbounded domains (\Cref{rem:lipschitz-implies-increment}), and we show that some distributional assumption is required for finite sample-complexity in this setting (\Cref{rem:sc-increment-necessity}). Practically relevant problems like regularized logistic regression satisfy this condition (\Cref{rem:sc-logistic-example}) and cannot be directly handled with the older assumptions from the literature. Here we prove another separation phenomenon: while uniform convergence is not possible even with this sub-Gaussian assumption~\Cref{thm:sc-anchored-uc}, 
ERM solves the unbounded integer problem with the optimal number of samples (\Cref{thm:sc-erm-upper-lower}). Even over bounded domains, ERM has strictly better sample complexity than uniform convergence in the regime where the condition number of the loss function is smaller than $\lfloor R\rfloor^2$. This is the only situation out of all those we consider in which we see a separation between uniform convergence and ERM. This is summarized in the last row of Table~\ref{table:results}.

\old{
\paragraph{Strongly convex, smooth objectives.} In continuous convex optimization, one can obtain better sample complexity bounds by imposing smoothness and strong convexity assumptions on the loss function~\cite[Chapter 8]{shapiro2021lectures}. We also investigate such loss functions with integrality constraints, with a view to compare with the corresponding results in the continuous world. %
In prior work on strongly convex objectives (see~\cite{harvey2019tight} and the references therein), the constraint set $X$ is assumed to be a compact convex set and $f(\x\,; z)$ is assumed to be $1$-Lipschitz over $X$ for all $z \in \mathcal{Z}$. This is equivalent to making an assumption on the distribution $D$ so that it is supported on $z \in \mathcal{Z}$ such that $f(\x\,; z)$ is $1$-Lipschitz over $X$. Some distributional assumption is necessary beyond just restricting to strongly convex and smooth functions, because otherwise one cannot solve~\eqref{eq:SCO} with finitely many samples\footnote{As far as we are aware, a formal proof of this necessity for a distributional assumption does not appear explicitly anywhere in the literature. In this paper, we give a short rigorous argument for this in Remark~\ref{rem:sc-increment-necessity}.}. %
We make a weaker distributional assumption to establish our results. In particular, any distribution supported on $z \in \mathcal{Z}$ such that $f(\x\,; z)$ is $1$-Lipschitz over a compact convex set $X$ will satisfy our assumptions, but our assumption is satisfied for a strictly larger class of distributions that includes practically relevant problems like regularized logistic regression that cannot be directly handled with the older assumptions from the literature; see Remark~\ref{rem:sc-logistic-example}. %
Under this assumption, we first prove a lower bound of $\Omega\left(\frac{1}{\epsilon^2}\right)$ samples for optimizing over the integer points in an $\ell_\infty$, as a function of the error $\epsilon$; in comparison, it is well-known that with continuous variables, the sample complexity is $O\left(\frac{1}{\epsilon}\right)$ -- this first appeared in~\cite{nemirovski2009robust};  see~\cite{harvey2019tight} for a discussion of follow up work. Thus, integer optimization with strongly convex, smooth objectives has a {\em much higher sample complexity} compared to continuous optimization. A second interesting phenomena that we uncover here is that ERM solves the integer problem with the optimal number of samples with a matching upper bound (Theorem~\ref{thm:sc-erm-upper-lower}), which is strictly better than what is needed for uniform convergence in the integer setting (Theorem~\ref{thm:sc-anchored-uc}). This is summarized in the last row of Table~\ref{table:results}.
}

\paragraph{Related work.} Most work on sample complexity of stochastic optimization has focused on the convex setting, i.e., $f(\cdot\,; z)$ is convex for all $z \in \mathcal{Z}$ and $X$ is a closed, convex set. As discussed above, the results in~\cite{shalev2009stochastic,feldman2016generalization,carmon2024sample} clear up the picture for Lipschitz convex functions; see Appendix~\ref{sec:SCO-summary} for a detailed discussion. \citep{agarwal2009information} provides information-theoretic lower bounds for stochastic convex optimization and identifies the roles of dimension, norm geometry, strong convexity, and sparsity. Work in the strongly convex setting, with and without smoothness assumptions, is summarized in~\cite{harvey2019tight,harvey2019simple}. See also~\cite[Chapter 4]{lan2020first} for a textbook exposition of stochastic convex optimization.

In the nonconvex setting, research effort has concentrated on methods to find stationary points or local minima. State-of-the-art results are summarized in~\cite[Chapter 6]{lan2020first}. Our focus in this paper is different from most work in the nonconvex stochastic optimization literature, and closer to work cited above in the convex setting, since we focus on global minimizers. Moreover, we consider the setting where the samples reveal complete information about $f(\cdot\,; z_i)$; in contrast, most work in the nonconvex regime assumes only pointwise (subgradient and function value) access to the stochastic loss function.

There is also an enormous literature on empirical optimization in stochastic programming and stochastic discrete optimization, often through asymptotic, gap-dependent, or model-specific guarantees. The textbook~\cite{shapiro2021lectures} is a good summary of these developments. We do not discuss these lines of work in any further detail as they are not directly relevant to our emphasis on finite-sample bounds.  %

\section{Preliminaries}
\label{sec:setup}

In this section we define some basic notions from stochastic convex optimization and sample complexity. We assume that $f:X\times \mathcal{Z}\to \R$ from~\eqref{eq:SCO} is such that $z\mapsto f(\x;z)$ is measurable for every $\x\in X$. For any probability distribution $D$ on $\mathcal{Z}$ satisfying $\E_{z\sim D}[f(\x;z)]\in\R$, we call the function $F_D$ defined in~\eqref{eq:SCO} as the {\em population objective or loss}, and the function $F_S$ defined in~\eqref{eq:ERM-SAA} based on a sample $S=(z_1,\ldots,z_m)\sim D^m$ as the {\em empirical objective or loss}. 
We write
\[
\hat \x_S\in \argmin_{\x\in X} F_S(\x),\qquad \x_D^\star\in \argmin_{\x\in X} F_D(\x)
\]
for the selected empirical and population minimizers. We assume throughout that the minimizers of both $F_D$ and $F_S$ exist (for all samples $S$). In this paper, this is without loss of generality since $f(\cdot\,; z)$ is continuous on its domain for all $z\in \mathcal{Z}$; in \Cref{sec:linfty-ball,sec:l2-ball} this implies the existence of minimizers because of compactness of the ambient balls, and in \Cref{sec:smooth-strongly-convex} it follows from strong convexity. %

Given a family $\mathfrak{D}$ of distributions on $\mathcal{Z}$ and parameters $\epsilon>0$ and $\delta\in(0,1)$, the sample complexities for uniform convergence, ERM, and the best proper algorithm are defined, respectively, as\footnote{We use the standard measurability convention that the uniform-deviation and excess events are measurable for the selected ERM output and for the learning rules considered.}
\begin{gather*}
m_{X}^{\mathrm{UC}}(\epsilon,\delta;\mathfrak{D}) := \inf\left\{m\in \N:\ \forall D\in \mathfrak{D},\ \Prob_{S\sim D^m}\left[\sup_{\x\in X}\left|F_D(\x)-F_S(\x)\right|\le \epsilon\right]\ge 1-\delta\right\},\\[0.3em]
m_{X}^{\mathrm{ERM}}(\epsilon,\delta;\mathfrak{D}) := \inf\bigg\{m\in \N:\ \forall D\in \mathfrak{D},\ \Prob_{S\sim D^m}\Big[F_D(\hat \x_S)-F_D(\x^\star_D) \le \epsilon\Big]\ge 1-\delta\bigg\}, \\[0.3em]
m_{X}^\star(\epsilon,\delta;\mathfrak{D}) :=\inf\left\{m\in \N:
\begin{array}{l}
\exists A_m:\mathcal{Z}^m\to X \text{ such that }
\forall D\in \mathfrak{D}, \\
\Prob_{S\sim D^m}\left[F_D(A_m(S))-F_D(\x^\star_D)\le \epsilon\right]\ge 1-\delta
\end{array}
\right\}.
\end{gather*}

We use the convention that the infimum over the empty set equals $+\infty$.

Notice that optimization is not affected by shifting functions by a constant value, that is, the optimal solution of $g(\x)$ is the same as that of $g(\x) + c$. While ERM sample complexity satisfies this shift-invariance, since it works directly with the optimal population/empirical solution, the standard uniform convergence definition above does \textbf{not} satisfy this natural invariance. Indeed, it is easy to see that if the family $\mathfrak{D}$ contains for all $p \in (0,1), M \ge 0$ the distributions that put mass $p$ on the constant function equal to $M$ everywhere and mass $1-p$ on the all-zeros function, then $m_{X}^{\mathrm{UC}}(\epsilon,\delta;\mathfrak{D}) = +\infty$.

Thus, it will be convenient to define the \emph{anchored uniform convergence (a-UC)} sample complexity where we ``quotient out'' all constant function shifts and anchor the functions $f(\cdot\,;z)$ at the origin. More precisely, define $\bar{f} : X \times \mathcal{Z}$ where $\bar{f}(\cdot\,;z) := f(\cdot\,;z) - f(\0; z)$, so that $\bar{f}(\0;z) = 0$ for all $z \in \mathcal{Z}$. We define the population and empirical objectives $\bar{F}_D$ and $\bar{F}_S$ with respect to $\bar{f}$ in the natural way, and the anchored uniform convergence sample complexity is given by 
\begin{align*}
m_{X}^{\mathrm{a-UC}}(\epsilon,\delta;\mathfrak{D}) := \inf\left\{m\in \N:\ \forall D\in \mathfrak{D},\ \Prob_{S\sim D^m}\left[\sup_{\x\in X}\left|\bar{F}_D(\x)-\bar{F}_S(\x)\right|\le \epsilon\right]\ge 1-\delta\right\},
\end{align*}

We note the standard domination relationship relationship between these notions: (anchored) uniform convergence is stronger than ERM, and ERM itself is a particular algorithm for processing the data sample and allowing other data-driven algorithms can reduce the sample complexity.

\begin{obs}\label{obs:sampleComplexity} For every $\epsilon>0$, every $\delta\in(0,1)$, and every family $\mathfrak{D}$, we have
\[
m_{X}^\star(\epsilon,\delta;\mathfrak{D})
\le m_{X}^{\mathrm{ERM}}(\epsilon,\delta;\mathfrak{D})
\le m_{X}^{\mathrm{a-UC}}(\epsilon/2,\delta;\mathfrak{D})
\]
\end{obs}

\section{Lipschitz Stochastic Optimization over a Subset of a Box}
\label{sec:linfty-ball}

In this section we show that when the domain is the $\ell_\infty$-ball, there is no separation between convex vs nonconvex and continuous vs discrete Lipschitz  stochastic optimization.

More precisely, we focus on the domain being the $\ell_\infty$-ball $B_R^{(\infty)}$ of radius $R$.  
Throughout this section, we equip $\R^d$ with $\ell_\infty$ norm, and consider a loss map  $f:B_R^{(\infty)}\times\mathcal{Z}\to\R$ such that each function $\x \mapsto f(\x;z)$ is $1$-Lipschitz (not necessarily convex)  on $B_R^{(\infty)}$ with respect to $\ell_\infty$, for every $z\in\mathcal{Z}$. 

\begin{theorem}\label{thm:upper-bound}
Suppose $f:B_R^{(\infty)}\times\mathcal{Z}\to\R$ is such that $f(\cdot\,;z)$ is $1$-Lipschitz with respect to $\ell_\infty$ for every $z\in\mathcal{Z}$. Then there exists an absolute constant $C>0$ such that for every family $\mathfrak{D}$ of distributions, every $\epsilon>0$, and every $\delta\in(0,1)$,
\[
m_{B_R^{(\infty)}}^{\mathrm{a-UC}}(\epsilon,\delta;\mathfrak{D})\le \left\lceil C\frac{R^2}{\epsilon^2}\left(d+\log\frac{1}{\delta}\right)\right\rceil.
\]
Moreover, there exist absolute constants $c,c_0>0$ such that, for every $d,R\ge1$, there is a family $\mathfrak{D}^{\mathrm{lin}}_{R}$ of distributions over $1$-Lipschitz linear functions such that, for every $0<\epsilon\le c_0R$ and every $0<\delta\le1/4$,
\[
m_{B_R^{(\infty)}}^\star(\epsilon,\delta;\mathfrak{D}^{\mathrm{lin}}_{R})
\ge c\frac{R^2}{\epsilon^2}\left(d+\log\frac{1}{\delta}\right).
\]
\end{theorem}

The formal proof of Theorem~\ref{thm:upper-bound} is given in \Cref{app:linfty-proofs}, but we sketch the main ideas in the proof here.

The upper bound follows from a bracketing and chaining argument for the anchored loss class. After anchoring at the origin, Lipschitzness gives $|f(\x;z) - f(\0;z)|\le R$ on $B_R^{(\infty)}$, and the $\ell_\infty$ covering entropy of the box is of order $d$. This gives anchored UC with the stated sample size, and hence the same upper bound for every feasible subset of the box.

The lower bound already holds for stochastic linear optimization with only bound constraints. The hard instances are indexed by a hidden sign vector $\b\in\{\pm1\}^d$, with population objective $F_\b(\x)=-(\rho/d)\langle \b,\x\rangle$. Thus the unknown distribution determines which vertex $R\b$ of the box is exposed by the linear objective. An $\epsilon$-optimal output must have enough average alignment with this exposed vertex. Each sample is a coordinate-wise biased coin experiment: it first chooses a coordinate $j\in[d]$ uniformly at random, and then returns a sign with mean $\rho b_j$. Therefore, unless $m$ is of order $d/\rho^2$, no proper algorithm can learn the objective direction well enough to output a near-optimal point. Taking $\rho$ of order $\epsilon/R$ gives the dimension term, and the one-coordinate coin flip gives the $\log(1/\delta)$ confidence term. Unlike the ERM lower bounds in~\cite{feldman2016generalization}, this construction does not rely on a bad empirical minimizer. The ambiguity is in the population linear objective itself, and hence the lower bound applies to every proper algorithm. 

We remark that since the lower bound uses only linear functions, it actually also holds when $X = B_R^{(\infty)}\cap\Z^d=\{-\lfloor R\rfloor,\ldots,\lfloor R\rfloor\}^d$, with $R$ replaced by $\lfloor R\rfloor$. Together with the upper bound from \Cref{thm:upper-bound}, this shows that continuous and integral feasible regions, as well as convex and nonconvex losses, have the same sample complexity in this $\ell_\infty$ setting.

\section{Lipschitz Stochastic Optimization over a Subset of a Euclidean Ball}
\label{sec:l2-ball}

In the previous section we observe that when the domain is the $\ell_\infty$-ball, different variants of Lipschitz  stochastic optimization collapsed to the easiest setting: linear optimization with variable constraints. We now show that when the domain is the $\ell_2$-ball, new phenomena occur and there is a strict separation between convex vs nonconvex functions and continuous vs integral feasible sets. 

More precisely, we now consider the domain being the Euclidean ball $B_R^{(2)}$ of radius $R$,
equip $\R^d$ with the Euclidean norm, and consider a loss map  $f:B_R^{(2)}\times\mathcal{Z}\to\R$ such that each function $\x \mapsto f(\x;z)$ is $1$-Lipschitz (not necessarily convex) on $B_R^{(2)}$ with respect to $\ell_2$, for every $z\in\mathcal{Z}$. 

We start by obtaining tight sample complexity bounds for the continuous feasible set, nonconvex function setting. The proofs for this section are given in \Cref{app:l2-proofs}.

\begin{theorem}\label{thm:l2-continuous-nonconvex-upper-lower}
Suppose $f:B_R^{(2)}\times\mathcal{Z}\to\R$ is such that $f(\cdot\,;z)$ is $1$-Lipschitz with respect to $\ell_2$ for every $z\in\mathcal{Z}$. Then there exists an absolute constant $C>0$ such that for every family $\mathfrak{D}$ of distributions, every $\epsilon>0$, and every $\delta\in(0,1)$,
\[
m_{B_R^{(2)}}^{\mathrm{a-UC}}(\epsilon,\delta;\mathfrak{D})
\le \left\lceil C\frac{R^2}{\epsilon^2}\left(d+\log\frac{1}{\delta}\right)\right\rceil.
\]
Moreover, there exist absolute constants $c, c_0>0$ such that, for every $d\ge1$ and $R>0$, there is a family $\mathfrak{D}$ of probability distributions over functions that are $1$-Lipschitz with respect to $\ell_2$ such that for every $0<\epsilon\le c_0R$ and every $0<\delta\le1/4$,
\[
m_{B_R^{(2)}}^\star(\epsilon,\delta;\mathfrak{D})\ge c\frac{R^2}{\epsilon^2}\left(d+\log\frac{1}{\delta}\right).
\]
\end{theorem}

This is in contrast to the convex Lipschitz setting discussed in the Introduction, where the dimension-dependent term in the ERM sample complexity is only of order $Rd/\epsilon$ (see Table~\ref{table:SCO} in Appendix~\ref{sec:SCO-summary}). The upper bound follows from the same bracketing and chaining argument used in \Cref{thm:upper-bound}. The lower bound uses the additional freedom provided by nonconvexity. We take an exponential-size packing $\W\subseteq B_R^{(2)}$ with pairwise distances of order $R$, and place a small $1$-Lipschitz tent around each point of $\W$, with disjoint supports. A distribution then hides one center $\u\in\W$ by biasing the activation probability of the tent at $\u$. The population objective has its unique minimizer at this hidden center, and every $\epsilon$-optimal proper output must fall in the corresponding tent. Thus optimization reduces to identifying the hidden element of $\W$ from biased Bernoulli observations. Since $\log|\W|=\Theta(d)$ and the bias is of order $\epsilon/R$, Fano's inequality and a two-point testing argument give the two terms in the lower bound.

\medskip

We now consider integral feasible sets, that is, subsets of $$X_R^{(2)}:=B_R^{(2)}\cap\Z^d.$$ We show that stochastic optimization over these smaller feasible sets can have a provably smaller sample complexity. The fact that this phenomenon happens in the $\ell_2$ domain, and not on the $\ell_\infty$ domain from the previous section is due to the mismatch between the geometry of the integer lattice and that of the standard Euclidean ball. More concretely, integer points in a Euclidean ball have sparse support: at most $\lfloor R^2\rfloor$ coordinates can be nonzero. When $R^2\ll d$, this sparsity makes the combinatorial scale much smaller than that of the full integer cube.

To capture this, define $H_2(d,R)$ by $H_2(d,R)=0$ for $0<R<1$, and
\[
H_2(d,R):=\min\{d,\lfloor R^2\rfloor\}\log\left(ed/{\min\{d,\lfloor R^2\rfloor\}}\right)
\]
for $R\ge1$. When $1\le\lfloor R^2\rfloor<d$, this becomes $H_2(d,R) = \lfloor R^2\rfloor \log (ed / \lfloor R^2\rfloor)$ and corresponds to (as shown in \Cref{app:l2-proofs}) the log number of integer points in the ball $B^{(2)}_{R}$, while when $\lfloor R^2\rfloor\ge d$, we have $H_2(d,R)=d$, which corresponds to the usual $d$-dimensional Euclidean covering scale. We show that this term $H_2(d,R)$ precisely characterizes the sample complexity of optimization over subsets of the integer set $X^{(2)}_R$.

\begin{theorem}\label{thm:l2-uc-upper}
Suppose the functions $f(\cdot\,;z)$ are 1-Lipschitz with respect to $\ell_2$ for all $z \in \mathcal{Z}$. Then there exists an absolute constant $C>0$ such that for every family $\mathfrak{D}$ of distributions, every $\epsilon>0$, and every $0<\delta<1/2$,
\[
m_{X_R^{(2)}}^{\mathrm{a-UC}}(\epsilon,\delta;\mathfrak{D}) \le \left\lceil C\frac{R^2}{\epsilon^2}\left(H_2(d,R)+\log\frac{1}{\delta}\right)\right\rceil.
\]
Moreover, there exist absolute constants $c,c_0>0$ such that, for every $d\ge1$ and $R\ge1$, there is a family of distributions $\mathfrak{D}$ on 1-Lipschitz with respect to $\ell_2$ such that for every $0<\epsilon\le c_0R$ and every $0<\delta\le1/4$,
\[
m_{X_R^{(2)}}^\star(\epsilon,\delta;\mathfrak{D}) \ge c\frac{R^2}{\epsilon^2}\left(H_2(d,R)+\log\frac{1}{\delta}\right).
\]
\end{theorem}

Thus, over the Euclidean integer ball, anchored UC, ERM, and the best algorithm all have the same sample complexity, up to constants. The rate is governed by $H_2(d,R)$ that encodes the difference between geometry of the $\ell_2$ ball and the integer points in it, and is strictly smaller than the continuous nonconvex rate when $\lfloor R^2 \rfloor < d$. We next show that this statistical advantage of integrality disappears, and in fact reverses, under smoothness and strong convexity.

\section{Stochastic Optimization of Strongly Smooth/Convex Functions}
\label{sec:smooth-strongly-convex}

We next consider stochastic optimization of smooth, strongly convex losses, both in unbounded domains and also on a box domain. The main result is that in this setting there is a significant qualitative separation in the sample complexity in the scenarios considered: First, for unbounded integer constrained problems, anchored uniform convergence has infinite sample complexity (\Cref{thm:sc-anchored-uc}) while ERM has finite sample complexity (\Cref{thm:sc-erm-upper-lower}); this happens because ERM witnesses a \emph{localization} phenomenon, where it suffices to control the stochastic process around the optimum. Moreover, unlike in the previous section, we prove that integer constrained optimization is provably harder than continuous optimization. More precisely, while in this integer constrained setting ERM has sample complexity of order $\frac{1}{\epsilon^2}$, we show that in the continuous case it has ``accelerated'' sample complexity of order $\frac{1}{\epsilon}$ (\Cref{thm:sc-continuous-erm}). 
Actually, as we show in \Cref{rem:sc-increment-necessity}, some assumption on the function distribution is required for any learning with finitely many samples in the unbound domain case, and thus we work throughout this section under a natural sub-gaussian increment assumption discussed in the sequel.

More precisely, in this section the loss map is defined on all of $\R^d$, while the parameter $R\in[1,\infty]$ specifies the feasible box $B_R^{(\infty)}=\{\x\in\R^d:\ \|\x\|_\infty\le R\}$, with the convention that $B_\infty^{(\infty)}=\R^d$. Fix $0<\mu\le L$ and set $\kappa:=L/\mu$. Recall that a function $h:\R^d\to\R$ is \emph{$\mu$-strongly convex} and \emph{$L$-smooth} if $h-\frac{\mu}{2}\|\cdot\|_2^2$ and $\frac{L}{2}\|\cdot\|_2^2-h$ are convex. Equivalently, whenever $h$ is differentiable,
\begin{equation}\label{eq:sc-first-order}
\frac{\mu}{2}\|\x-\y\|_2^2
\le h(\y)-h(\x)-\langle\nabla h(\x),\y-\x\rangle
\le \frac{L}{2}\|\x-\y\|_2^2
\qquad \forall \x,\y\in\R^d.
\end{equation}
We consider loss maps $f:\R^d\times\mathcal{Z}\to\R$ such that $f(\cdot\,;z)$ is $\mu$-strongly convex and $L$-smooth for every $z\in\mathcal{Z}$. Since the defining convexity conditions are preserved under averaging, the population and empirical objectives $F_D$ and $F_S$ satisfy the same strong-convexity and smoothness bounds whenever they are finite.

\paragraph{Sub-Gaussian Increment Condition.} We say that a distribution over $\mathcal{Z}$ satisfies \emph{$\sigma$-sub-Gaussian increment condition over a domain $X \subseteq \R^d$} if for every $\x,\y\in  X$ and every $\lambda\in\R$,
\begin{align}
\E_{z\sim D}\exp\left(\lambda\left[f(\x;z)-f(\y;z)-F_D(\x)+F_D(\y)\right]\right)
\le \exp\left(\frac{\lambda^2\sigma^2\|\x-\y\|_2^2}{2}\right). \label{eq:subgaussInc}
\end{align}
Related assumptions on centered differences of sample losses  appear, for example, in~\citep[Theorem~5]{chu2023unified}.
We use $\mathfrak{D}^{\mathrm{sc}}_{\mu,L,\sigma}$ to denote the class of all distributions satisfying this sub-Gaussian increment condition.  
In sample complexity statements over a feasible set $X$, we interpret this class relative to the same set: \eqref{eq:subgaussInc} is imposed only for $\x,\y\in X$, and no increment control outside $X$ is required. Note that the increment condition controls differences of losses rather than absolute offsets, and so remains meaningful on an unbounded domain. %

\begin{remark}[Necessity of increment condition] \label{rem:sc-increment-necessity}
Some centered tail control is necessary in the unbounded setting. Without it, smoothness and strong convexity alone do not allow proper learning with finite sample complexity, already for one-dimensional quadratic losses $f(x;z) = \frac{\mu}{2}(x - z)^2$. To see this, fix $\epsilon>0$, $\delta\in(0,1)$, number of samples $m\in\N$, and an arbitrary proper algorithm $A_m:\R^m\to \R$; we show that with probability greater than $\delta$ this algorithm has excess loss more than $\epsilon$ for some distribution. Choose $p\in(\delta^{1/m},1)$, and construct the distribution $D$ satisfying $\Prob_{Z \sim D}[Z = 0] = p$ and $\Prob_{Z \sim D}[Z = (A_m(\0) + \lceil2\sqrt{\epsilon/\mu}\rceil)/(1-p)] = 1-p$. The population loss $F_D$ is minimized at $\E Z = A_m(\0) + \left\lceil2 \sqrt{\epsilon/\mu}\right\rceil$, but with probability $p^m > \delta$ all samples are 0, so the algorithm outputs $A_m(\0)$ and incurs an excess loss of at least $\frac{\mu}{2} (2\sqrt{\epsilon/\mu})^2 = 2 \epsilon > \epsilon$. 
\end{remark}

\begin{remark}[Lipschitzness implies the increment condition]\label{rem:lipschitz-implies-increment}
Suppose that $f(\cdot\,;z)$ is $L$-Lipschitz on $X$ for every $z\in\mathcal{Z}$, and let $D$ be any distribution on $\mathcal{Z}$. Fix $\x,\y\in X$ and put $Y:=f(\x;Z)-f(\y;Z)$ for $Z\sim D$. Then $|Y|\le L\|\x-\y\|_2$, and Hoeffding's lemma~\citep[Lemma~B.7]{shalev2014understanding} gives
\[
\E\exp(\lambda(Y-\E Y))\le \exp\left(\frac{\lambda^2L^2\|\x-\y\|_2^2}{2}\right)\qquad \forall \lambda\in\R.
\]
Thus $D$ satisfies the sub-Gaussian increment condition over $X$ with $\sigma=L$.
\end{remark}

\begin{remark}[Regularized Logistic Loss]\label{rem:sc-logistic-example}
\Cref{rem:lipschitz-implies-increment} gives one sufficient condition for the increment assumption, but the condition can hold even when the sample losses are not Lipschitz. We note that this happens for regularized logistic loss when the feature vectors are bounded. This setting corresponds to the loss map
\[
f(\x;(\a,b))=\frac{\mu}{2}\|\x\|_2^2+\log(1+\exp(-b\langle \a,\x\rangle)),
\]
where $b\in\{\pm1\}$ and $\|\a\|_2\le M$ for some $M$. Since $u\mapsto\log(1+\exp(-u))$ is $1$-Lipschitz and $1/4$-smooth, $f(\cdot\,;(\a,b))$ is $\mu$-strongly convex and $(\mu+M^2/4)$-smooth. Also, for any $\x,\y\in\R^d$,
\[
|\log(1+\exp(-b\langle \a,\x\rangle))-\log(1+\exp(-b\langle \a,\y\rangle))|
\le M\|\x-\y\|_2.
\]
The quadratic term is deterministic and cancels after centering. Hoeffding's Lemma therefore gives the increment condition with $\sigma=M$. The same argument applies if the logistic term is replaced by $\ell(-b\langle \a,\x\rangle)$, where $\ell:\R\to\R$ is a convex margin loss that is $\rho$-Lipschitz and $\beta$-smooth. In that case one may take $L=\mu+\beta M^2$ and $\sigma=\rho M$.
\end{remark}

\paragraph{Sample complexity of a-UC vs ERM and integer vs continuous.} We now turn to our results about sample complexity in this setting. Let $X_R^{(\infty)} := B_R^{(\infty)}\cap\Z^d$ denote the set of integer points in the box of radius $R$. The proofs for this section are given in \Cref{app:sc-proofs}.

Our first theorem proves a tight bound on the anchored uniform convergence for the integer domain $X^{(\infty)}_R$. Crucially, the lower bound shows that this complexity goes to infinity as the box size increases.

\begin{theorem}\label{thm:sc-anchored-uc}
Suppose the functions $f(\cdot\,;z)$ are $\mu$-strongly convex and $L$-smooth for all $z \in \mathcal{Z}$. Then there is an absolute constant $C>0$ such that, for every $1\le R<\infty$, every $\epsilon>0$, and every $\delta\in(0,1)$,
\[
m_{X_R^{(\infty)}}^{\mathrm{a-UC}}(\epsilon,\delta;\mathfrak{D}^{\mathrm{sc}}_{\mu,L,\sigma})
\le \left\lceil C\frac{\sigma^2\lfloor R\rfloor^2 d}{\epsilon^2}\left(d+\log\frac{1}{\delta}\right)\right\rceil.
\]
Moreover, there is an absolute constant $c>0$ such that, for every $d\ge1$ and every $1\le R<\infty$, there is such a loss map $f(\cdot\,;\,\cdot)$ for which, for every $0<\delta\le1/4$ and every $\epsilon>0$,
\[
m_{X_R^{(\infty)}}^{\mathrm{a-UC}}(\epsilon,\delta;\mathfrak{D}^{\mathrm{sc}}_{\mu,L,\sigma})
\ge c\frac{\sigma^2\lfloor R\rfloor^2 d}{\epsilon^2}\left(d+\log\frac{1}{\delta}\right).
\]
\end{theorem}

In contrast, we now show that ERM is qualitatively stronger in this setting: the sample complexity remains finite even when considering optimization over all the integers $X^{(\infty)}_\infty := \Z^d$, i.e., when the box radius $R = \infty$.

\begin{theorem}\label{thm:sc-erm-upper-lower}
Suppose the functions $f(\cdot\,;z)$ are $\mu$-strongly convex and $L$-smooth for all $z \in \mathcal{Z}$. Then there exists an absolute constant $C>0$ such that, for every $R\in[1,\infty]$, every $0<\epsilon\le \mu d$, and every $\delta\in(0,1)$,
\[
m_{X_R^{(\infty)}}^{\mathrm{ERM}}(\epsilon,\delta;\mathfrak{D}^{\mathrm{sc}}_{\mu,L,\sigma})
\le \left\lceil C\frac{\sigma^2d\min\{\kappa,\lfloor R\rfloor^2\}}{\epsilon^2}\left(d+\log\frac{1}{\delta}\right)\right\rceil.
\]
Moreover, there exist absolute constants $c,c_0>0$ such that, for every $d\ge2$ and every $R\in[1,\infty]$, there is such a loss map $f(\cdot\,;\,\cdot)$ for which, for every $0<\delta\le1/4$ and every $0<\epsilon\le c_0\mu d$,
\[
m_{X_R^{(\infty)}}^\star(\epsilon,\delta;\mathfrak{D}^{\mathrm{sc}}_{\mu,L,\sigma})\ge c\frac{\sigma^2d\min\{\kappa,\lfloor R\rfloor^2\}}{\epsilon^2}\left(d+\log\frac{1}{\delta}\right).
\]
\end{theorem}

As mentioned before, this difference comes from the fact that ERM in this setting exhibits a localization phenomenon, requiring empirical-process control only on a (random) region near the optimal solution; in fact, strong convexity localizes this region to scale $\sqrt{\kappa d}$, truncated by the box diameter. Note also that the lower bound in Theorem~\ref{thm:sc-erm-upper-lower} is for any algorithm; thus, no algorithm can improve upon ERM (up to constants).

Finally, we consider the continuous version of the problem, namely when the domain is the whole of $\R^d$. In this case, we show that ERM actually has ``accelerated'' rates of order $\frac{1}{\epsilon}$, instead of order $\frac{1}{\epsilon^2}$ that is best possible for the integer case, as shown by the previous theorem. This happens because in the continuous case we can avoid an additional rounding loss in the localization phenomenon, which is unavoidable in the integer case (due to our matching lower bound).

\begin{theorem}\label{thm:sc-continuous-erm}
Suppose the functions $f(\cdot\,;z)$ are $\mu$-strongly convex and $L$-smooth for all $z \in \mathcal{Z}$. Then there exists an absolute constant $C>0$ such that, for every nonempty closed convex set $K\subseteq\R^d$, every $\epsilon>0$, and every $\delta\in(0,1)$,
\[
m_K^{\mathrm{ERM}}(\epsilon,\delta;\mathfrak{D}^{\mathrm{sc}}_{\mu,L,\sigma})
\le \left\lceil C\frac{\sigma^2}{\mu\epsilon}\left(d+\log\frac{1}{\delta}\right)\right\rceil.
\]
\end{theorem}

We note that this statement does not follow directly from existing SGD bounds of order $\frac{1}{\epsilon}$ (e.g., \cite{nemirovski2009robust} and \citet[Theorems~3.1 and~3.5]{harvey2019tight}: their results assume an objective that is both $1$-strongly convex and $1$-Lipschitz on the feasible set. When the feasible set is unbounded, however, this class is empty, since no finite-valued function on all of $\R^d$ can be both strongly convex and globally Lipschitz. While the Lipschitz condition is generalized by our sub-Gaussian increment condition (Remark~\ref{rem:lipschitz-implies-increment}), there is no existing result in the continuous setting under the weaker sub-Gaussian increment condition. Thus, we prove Theorem~\ref{thm:sc-continuous-erm} to compare to the integer case with the same assumptions, even on unbounded domains.

\begin{ack}
Hongyu Cheng and Amitabh Basu gratefully acknowledge support from the Air Force Office of Scientific Research (AFOSR) grant FA9550-25-1-0038. Hongyu Cheng also received support from a MINDS Fellowship awarded by the Mathematical Institute for Data Science (MINDS) at Johns Hopkins University.
\end{ack}

\bibliographystyle{plainnat}
\bibliography{references}

@book{SchrijverBOOK,
	Address = {Chichester},
	Author = {Schrijver, A.},
	Isbn = {0-471-90854-1},
	Mrclass = {90C05 (90C10)},
	Mrnumber = {874114 (88m:90090)},
	Mrreviewer = {J{\"u}rgen K{\"o}hler},
	Note = {A Wiley-Interscience Publication},
	Pages = {xii+471},
	Publisher = {John Wiley \& Sons Ltd.},
	Series = {Wiley-Interscience Series in Discrete Mathematics},
	Title = {Theory of linear and integer programming},
	Year = {1986}}

@book{shalev2014understanding,
  title={Understanding machine learning: From theory to algorithms},
  author={Shalev-Shwartz, Shai and Ben-David, Shai},
  year={2014},
  publisher={Cambridge University Press}
}

@book{basu2025convexity,
  title={Convexity and its applications in discrete and continuous optimization},
  author={Basu, Amitabh},
  year={2025},
  publisher={Cambridge University Press}
}

@book{vershynin2018high,
  title={High-dimensional probability: An introduction with applications in data science},
  author={Vershynin, Roman},
  volume={47},
  year={2018},
  publisher={Cambridge University Press}
}

@book{tsybakov2008introduction,
  title={Introduction to Nonparametric Estimation},
  author={Tsybakov, A.B.},
  isbn={9780387790527},
  lccn={2008939894},
  series={Springer Series in Statistics},
  year={2008},
  publisher={Springer New York}
}

@misc{duchi2025statistics,
  title={Statistics and Information Theory},
  author={Duchi, John},
  year={2025},
  url={https://stanford.edu/class/ee377/lecture-notes.pdf}
}

@article{chu2023unified,
  title={A unified framework for information-theoretic generalization bounds},
  author={Chu, Yifeng and Raginsky, Maxim},
  journal={Advances in Neural Information Processing Systems},
  volume={36},
  pages={79260--79278},
  year={2023}
}

@article{feldman2016generalization,
  title={Generalization of {ERM} in stochastic convex optimization: The dimension strikes back},
  author={Feldman, Vitaly},
  journal={Advances in Neural Information Processing Systems},
  volume={29},
  year={2016}
}

@inproceedings{harvey2019tight,
  title={Tight analyses for non-smooth stochastic gradient descent},
  author={Harvey, Nicholas JA and Liaw, Christopher and Plan, Yaniv and Randhawa, Sikander},
  booktitle={Conference on Learning Theory},
  pages={1579--1613},
  year={2019},
  organization={PMLR}
}

@inproceedings{shalev2009stochastic,
  title={Stochastic Convex Optimization},
  author={Shalev-Shwartz, Shai and Shamir, Ohad and Srebro, Nathan and Sridharan, Karthik},
  booktitle={Proceedings of the 22nd Annual Conference on Learning Theory},
  year={2009}
}

@book{conforti2014integer,
  title     = {Integer programming},
  author    = {Conforti, Michele and Cornu{\'e}jols, G{\'e}rard and Zambelli, Giacomo},
  volume    = {271},
  year      = {2014},
  publisher = {Springer},
  doi       = {10.1007/978-3-319-11008-0},
  url       = {https://doi.org/10.1007/978-3-319-11008-0}
}

@inproceedings{carmon2024sample,
  title={The sample complexity of {ERM}s in stochastic convex optimization},
  author={Carmon, Daniel and Yehudayoff, Amir and Livni, Roi},
  booktitle={International Conference on Artificial Intelligence and Statistics},
  pages={3799--3807},
  year={2024},
  organization={PMLR}
}

@article{agarwal2009information,
  title={Information-theoretic lower bounds on the oracle complexity of convex optimization},
  author={Agarwal, Alekh and Wainwright, Martin J and Bartlett, Peter and Ravikumar, Pradeep},
  journal={Advances in Neural Information Processing Systems},
  volume={22},
  year={2009}
}

@book{shapiro2021lectures,
  title={Lectures on stochastic programming: modeling and theory},
  author={Shapiro, Alexander and Dentcheva, Darinka and Ruszczynski, Andrzej},
  year={2021},
  publisher={SIAM}
}

@article{harvey2019simple,
  title={Simple and optimal high-probability bounds for strongly-convex stochastic gradient descent},
  author={Harvey, Nicholas JA and Liaw, Christopher and Randhawa, Sikander},
  journal={arXiv preprint arXiv:1909.00843},
  year={2019}
}

@article{nemirovski2009robust,
  title={Robust stochastic approximation approach to stochastic programming},
  author={Nemirovski, Arkadi and Juditsky, Anatoli and Lan, Guanghui and Shapiro, Alexander},
  journal={SIAM Journal on optimization},
  volume={19},
  number={4},
  pages={1574--1609},
  year={2009},
  publisher={SIAM}
}

@book{van1996weak,
  author={Van Der Vaart, Aad W and Wellner, Jon A},
  title={Weak convergence and empirical processes: with applications to statistics},
  year={1996},
  series={Springer Series in Statistics},
  publisher={Springer}
}

@book{lan2020first,
  title={First-order and stochastic optimization methods for machine learning},
  author={Lan, Guanghui},
  volume={1},
  year={2020},
  publisher={Springer}
}

\newpage
\appendix

\section{Sample complexity of stochastic convex optimization}\label{sec:SCO-summary}

\begin{table}[htbp]
\centering
\caption{Sample complexity of $1$-Lipschitz convex optimization in $\ell_2$}\label{table:SCO}
\begingroup
\scriptsize
\setlength{\tabcolsep}{2.2pt}
\renewcommand{\arraystretch}{1.45}
\newcommand{\ratecell}[1]{%
  \begingroup
  \setbox0=\hbox{$#1$}%
  \ifdim\wd0>\linewidth
    \makebox[\linewidth][c]{\resizebox{\linewidth}{!}{\box0}}%
  \else
    \makebox[\linewidth][c]{\box0}%
  \fi
  \endgroup}
\begin{tabular}{@{}>{\raggedright\arraybackslash}m{0.15\linewidth}>{\centering\arraybackslash}m{0.28\linewidth}>{\centering\arraybackslash}m{0.28\linewidth}>{\centering\arraybackslash}m{0.25\linewidth}@{}}
\toprule
Problem class & Any data-driven algorithm (SGD) & ERM & UC \\
\midrule
\begin{tabular}[c]{@{}l@{}}
$X = B_R^{(2)}$,  \\
$1$-Lipschitz convex $f$
\end{tabular}
&
\ratecell{\Theta\!\left(\tfrac{R^2}{\epsilon^2}\log(\frac{1}{\delta})\right)}
&
\ratecell{\Theta\!\left(\tfrac{Rd}{\epsilon}+\tfrac{R^2}{\epsilon^2}\log(\tfrac{1}{\delta}))\right)}
&
\parbox[c]{\linewidth}{\centering
\ratecell{\Theta\!\left(\tfrac{R^2}{\epsilon^2}(d+\log(1/\delta))\right)}%
}
\\
\bottomrule
\end{tabular}
\endgroup
\end{table}

Table~\ref{table:SCO} summarizes the precise sample complexity bounds when $\x \mapsto f(\x;z)$ is 1-Lipschitz with respect to $\ell_2$ over the $\ell_2$ ball. The sample complexity of uniform convergence is $\Theta\!\left(\tfrac{R^2}{\epsilon^2}(d+\log(\tfrac{1}{\delta}))\right)$; the upper bound is a consequence of our results in the second row of Table~\ref{table:results}, and the lower bound comes about by combining~\cite[Theorem 3.9]{feldman2016generalization} and the lower bound on uniform convergence in the first row of Table~\ref{table:results} for $B^{(\infty)}_R$ with $d=1$ (so both $\ell_2$ and $\ell_\infty$ norms are the same). The sample complexity of ERM is $\Theta\!\left(\tfrac{Rd}{\epsilon}+\tfrac{R^2}{\epsilon^2}\log(\tfrac{1}{\delta}))\right)$; the upper bound follows from Theorems 1 and 2 in~\cite{carmon2024sample} and the lower bound is from~\cite[Theorem 3.10]{feldman2016generalization}. Finally, one can use stochastic approximation (SA) methods or stochastic gradient descent (SGD) to obtain an $\epsilon$-approximate solution to~\eqref{eq:SCO} with probability $1-\delta$ with at most $O\left(\tfrac{1}{\epsilon^2}\log(\tfrac{1}{\delta})\right)$ samples; and no algorithm can do better because of the lower bound stated above for $B^{(\infty)}_R$ with $d=1$. Thus, over the $\ell_2$ ball, the sample complexity of uniform convergence for {\em convex}, Lipschitz functions is strictly larger than for ERM, which itself can be strictly improved by using SGD.

\section{Learning-theoretic tools}
\label[appendix]{app:learning-lemmas}

We first state the testing and KL facts used in the lower bounds.

\begin{lemma}\label{lem:testing-inequalities}
Let $\V$ be a finite set with $|\V|\ge2$, let $\{P_v:v\in\V\}$ be probability distributions on a common measurable space, and let $Q$ be another probability distribution on the same space. Then every hypothesis test or decorder $\widehat{v}$ with values in $\V$ satisfies
\[
\frac1{|\V|}\sum_{v\in\V}P_v[\widehat{v}\ne v]
\ge 1-\frac{|\V|^{-1}\sum_{v\in\V}\mathrm{KL}(P_v\|Q)+\log2}{\log|\V|}.
\]
Moreover, if $P$ and $Q$ are probability distributions on a common measurable space, $E$ is an event, and
\[
P(E)\ge1-\delta,\qquad Q(E)\le\delta
\]
for some $0<\delta\le1/4$, then
\[
\mathrm{KL}(P\|Q)\ge \frac12\log\frac{1}{2\delta}.
\]
\end{lemma}

\begin{proof}[Proof of \Cref{lem:testing-inequalities}]
For the first claim, let $V$ be uniform on $\V$ and, conditionally on $V=v$, let $Y$ have distribution $P_v$. Let $\bar P:=|\V|^{-1}\sum_{v\in\V}P_v$. By Fano's inequality \citep[Proposition~2.3.3 and Corollary~2.3.4]{duchi2025statistics} (see also \citealp[Section~2.7.1]{tsybakov2008introduction}),
\[
\frac1{|\V|}\sum_{v\in\V}P_v[\widehat{v}\ne v]\ge 1-\frac{I(V;Y)+\log2}{\log|\V|}.
\]
If $|\V|^{-1}\sum_v\mathrm{KL}(P_v\|Q)=+\infty$, the desired bound is trivial. Otherwise all $P_v$ are absolutely continuous with respect to $Q$, and the standard identity
\[
\frac1{|\V|}\sum_{v\in\V}\mathrm{KL}(P_v\|Q)
=I(V;Y)+\mathrm{KL}(\bar P\|Q)
\]
implies $I(V;Y)\le |\V|^{-1}\sum_v\mathrm{KL}(P_v\|Q)$.

For the second claim, apply the KL data processing inequality \citep[Proposition~2.2.13 and Corollary~2.2.14]{duchi2025statistics} to the indicator $y\mapsto \1_E(y)$. The random variable $\1_E(Y)$ has distribution $\mathrm{Bern}(P(E))$ when $Y\sim P$ and distribution $\mathrm{Bern}(Q(E))$ when $Y\sim Q$. Hence
\[
\mathrm{KL}(P\|Q)\ge \mathrm{kl}(P(E)\|Q(E)),
\]
where $\mathrm{kl}$ denotes the binary relative entropy. Since $P(E)\ge1-\delta$ and $Q(E)\le\delta$,
\[
\mathrm{kl}(P(E)\|Q(E))\ge \mathrm{kl}(1-\delta\|\delta)
=(1-2\delta)\log\frac{1-\delta}{\delta}
\ge \frac12\log\frac{1}{2\delta},
\]
where the last inequality uses $\delta\le1/4$.
\end{proof}

\begin{lemma}\label{lem:elementary-kl-bounds}\label{lem:bernoulli-kl}
The following elementary KL bounds hold.
\begin{enumerate}
\item For every $0\le \alpha\le1/4$,
\[
\mathrm{KL}\left(\mathrm{Bern}\left(\frac12+\alpha\right)\middle\|\mathrm{Bern}\left(\frac12-\alpha\right)\right)\le 16\alpha^2.
\]
Equivalently, if $P_+$ and $P_-$ are distributions on $\{\pm1\}$ with $P_+(s)=(1+\rho s)/2$ and $P_-(s)=(1-\rho s)/2$, then $\mathrm{KL}(P_+\|P_-)\le4\rho^2$ for $0\le\rho\le1/2$.
\item For every $k,m\in\N$, every $\theta,\theta'\in\R^k$, and every $\sigma>0$,
\[
\mathrm{KL}\left(N(\theta,\sigma^2I_k)^{\otimes m}\middle\|N(\theta',\sigma^2I_k)^{\otimes m}\right)=\frac{m}{2\sigma^2}\|\theta-\theta'\|_2^2.
\]
\end{enumerate}
\end{lemma}

\begin{proof}[Proof of \Cref{lem:elementary-kl-bounds}]
For the Bernoulli bound, we use
\[
\mathrm{KL}\left(\mathrm{Bern}\left(\frac12+\alpha\right)\middle\|\mathrm{Bern}\left(\frac12-\alpha\right)\right)
=2\alpha\log\frac{1+2\alpha}{1-2\alpha}
\le 16\alpha^2,
\]
where we used $\log((1+x)/(1-x))\le 4x$ for $0\le x\le1/2$, with $x=2\alpha$. The equivalent Bernoulli form follows by setting $\alpha=\rho/2$.
For the Gaussian identity, the one-sample formula
\[
\mathrm{KL}\left(N(\theta,\sigma^2I_k)\middle\|N(\theta',\sigma^2I_k)\right)=\frac{1}{2\sigma^2}\|\theta-\theta'\|_2^2
\]
follows by expanding the log-likelihood ratio. Product additivity of KL divergence gives the stated identity for product measures.
\end{proof}

\begin{definition}\label{def:covering-number}
Let $(T,\rho)$ be a metric space, let $A\subseteq T$, and let $r>0$. A finite set $C\subseteq T$ is an $r$-cover of $A$ with respect to $\rho$ if, for every $u\in A$, there exists $v\in C$ such that $\rho(u,v)\le r$. The covering number of $A$ at scale $r$ is
\[
N(r,A,\rho):=\min\left\{|C|:\ C\subseteq T,\ C \text{ is an } r\text{-cover of } A\right\}.
\]
If no finite $r$-cover exists, we set $N(r,A,\rho)=+\infty$. When $T$ is a normed space and $\rho(u,v)=\|u-v\|$, we also write $N(r,A,\|\cdot\|)$. For the Euclidean norm, we write $N(r,A):=N(r,A,\|\cdot\|_2)$.
\end{definition}

\begin{lemma}\label{lem:volumetric-cover}
Let $K\subseteq \R^d$ be nonempty and compact, let $\|\cdot\|$ be a norm on $\R^d$, and suppose that $\|\x\|\le R$ for every $\x\in K$. Then, for every $r>0$,
\[
N(r,K,\|\cdot\|)\le \left(1+\frac{2R}{r}\right)^d.
\]
\end{lemma}

\begin{proof}[Proof of \Cref{lem:volumetric-cover}]
Let $B:=\{\x\in\R^d:\|\x\|\le1\}$. If a finite set $P\subseteq K$ satisfies $\|\x-\y\|>r$ for every distinct $\x,\y\in P$, then the sets $\x+(r/2)B$, $\x\in P$, are pairwise disjoint and contained in $(R+r/2)B$. Hence
\[
|P|\left(\frac r2\right)^d\operatorname{vol}(B)\le \left(R+\frac r2\right)^d\operatorname{vol}(B),
\]
and so $|P|\le (1+2R/r)^d$. Thus every finite $r$-separated subset of $K$ has cardinality at most $(1+2R/r)^d$. Choose an $r$-separated subset $C\subseteq K$ of maximum cardinality. If $C$ were not an $r$-cover of $K$, a point of $K$ at distance greater than $r$ from every point of $C$ could be added to $C$, contradicting maximality. This proves the lemma.
\end{proof}

\begin{lemma}\label{lem:anchored-lipschitz-uc-vdw}
Let $K\subseteq\R^d$ be nonempty and compact, let $\|\cdot\|$ be a norm on $\R^d$, and suppose that $\0\in K$ and $\|\x\|\le R$ for every $\x\in K$, for some $R>0$. Let $f:K\times\mathcal{Z}\to\R$ be such that $f(\cdot\,;z)$ is $1$-Lipschitz with respect to $\|\cdot\|$ for every $z\in\mathcal{Z}$. Suppose that, for some $h\ge1$,
\[
N(\eta R,K,\|\cdot\|)\le \left(\frac{3}{\eta}\right)^h\qquad \forall\,0<\eta<1.
\]
Then there exists an absolute constant $C>0$ such that, for every family $\mathfrak{D}$ of distributions on $\mathcal{Z}$, every $\epsilon>0$, and every $\delta\in(0,1)$,
\[
m_K^{\mathrm{a-UC}}(\epsilon,\delta;\mathfrak{D})
\le \left\lceil C\frac{R^2}{\epsilon^2}\left(h+\log\frac{1}{\delta}\right)\right\rceil.
\]
\end{lemma}

\begin{proof}[Proof of \Cref{lem:anchored-lipschitz-uc-vdw}]
Fix $D\in\mathfrak{D}$. For $\x\in K$, define
\[
\ell_\x(z):=f(\x;z)-f(\0;z),\qquad g_\x(z):=\frac{\ell_\x(z)+R}{2R}.
\]
Then $0\le g_\x\le1$ and
\[
|g_\x(z)-g_\y(z)|\le \frac{\|\x-\y\|}{2R}\qquad \forall \x,\y\in K,\ z\in\mathcal{Z}.
\]
Let $\G_K:=\{g_\x:\x\in K\}$. If $C_\alpha$ is an $\alpha R$-cover of $K$, then the brackets
\[
\left[\max\{g_\y-\alpha/2,0\},\ \min\{g_\y+\alpha/2,1\}\right],\qquad \y\in C_\alpha,
\]
cover $\G_K$ and have $L_2(D)$-size at most $\alpha$. Hence, for every $0<\alpha<1$, we can bound the bracketing number (\cite[Definition 2.1.6]{van1996weak}) as follows:
\[
N_{[]}(\alpha,\G_K,L_2(D))\le N(\alpha R,K,\|\cdot\|)
\le \left(\frac{3}{\alpha}\right)^h.
\]
Let $P_m$ be the empirical measure of $S$, so that $P_m g=\frac{1}{m}\sum_{i=1}^m g(z_i)$ for every measurable $g$, and set $\mathbb G_m:=\sqrt m(P_m-D)$. By \cite[Theorem~2.14.9]{van1996weak}, applied with the bracketing bound above, there is an absolute constant $C_1>0$ such that
\[
\Prob\left[\|\mathbb G_m\|_{\G_K}>C_1\sqrt{h+\log\frac{1}{\delta}}\right]\le \delta.
\]
Thus, with probability at least $1-\delta$,
\[
\sup_{\x\in K}|(P_m-D)g_\x|\le C_1\sqrt{\frac{h+\log(1/\delta)}{m}}.
\]
Since
\[
\sup_{\x\in K}|\bar F_S(\x)-\bar F_D(\x)|
=2R\sup_{\x\in K}|(P_m-D)g_\x|,
\]
the desired anchored UC bound holds whenever
\[
m\ge 4C_1^2\frac{R^2}{\epsilon^2}\left(h+\log\frac{1}{\delta}\right).
\]
\end{proof}

\section{Proofs for 1-Lipschitz stochastic optimization over \texorpdfstring{$B_R^{(\infty)}$}{B infinity}}
\label[appendix]{app:linfty-proofs}

\subsection{Proof of \Cref{thm:upper-bound}}

By \Cref{lem:volumetric-cover}, for every $0<\eta<1$,
\[
N(\eta R,B_R^{(\infty)},\|\cdot\|_\infty)\le \left(1+\frac{2}{\eta}\right)^d\le \left(\frac{3}{\eta}\right)^d.
\]
\Cref{lem:anchored-lipschitz-uc-vdw}, applied with $K=B_R^{(\infty)}$, $\|\cdot\|=\|\cdot\|_\infty$, and $h=d$, gives the stated upper bound.

It remains to prove the lower bound. We prove it with $c_0=1/8$ and $c=1/1024$. Let $\mathcal{Z}:=[d]\times \{\pm 1\}$ and define
\[
f(\x;(j,k)):=-k x_j,\qquad \x\in B_R^{(\infty)},\ (j,k)\in \mathcal{Z}.
\]
For $\b\in\{\pm 1\}^d$ and $0<\rho\le 1/2$, let $D^{\mathrm{dim}}_{\b,\rho}$ be the distribution on $\mathcal{Z}$ defined by
\[
D^{\mathrm{dim}}_{\b,\rho}(j,k):=\frac{1}{d}\cdot \frac{1+\rho kb_j}{2} \qquad \forall j\in[d],\ k\in\{\pm 1\}.
\]
For $b\in\{\pm 1\}$ and $0<\rho\le 1/2$, let $D^{\mathrm{conf}}_{b,\rho}$ be the distribution on $\mathcal{Z}$ defined by
\[
D^{\mathrm{conf}}_{b,\rho}(1,k):=\frac{1+\rho kb}{2} \qquad \forall k\in\{\pm 1\}.
\]
All other pairs have probability zero. Set
\[
\mathfrak{D}^{\mathrm{dim}}_{R}:=\left\{D^{\mathrm{dim}}_{\b,\rho}:\b\in\{\pm 1\}^d,\ 0<\rho\le 1/2\right\},
\]
\[
\mathfrak{D}^{\mathrm{conf}}_{R}:=\left\{D^{\mathrm{conf}}_{b,\rho}:b\in\{\pm 1\},\ 0<\rho\le 1/2\right\},
\]
and $\mathfrak{D}^{\mathrm{lin}}_{R}:=\mathfrak{D}^{\mathrm{dim}}_{R}\cup \mathfrak{D}^{\mathrm{conf}}_{R}$.
For every $(j,k)\in \mathcal{Z}$, the map $\x\mapsto f(\x;(j,k))$ is linear and $1$-Lipschitz with respect to $\|\cdot\|_\infty$ on $B_R^{(\infty)}$.

We prove the dimension and confidence terms separately. Both parts reduce optimization to biased-coin hypothesis testing.

\smallskip
\noindent\emph{Dimension term.}
It suffices to prove the claim for $\delta=1/4$, since the confidence requirement becomes stronger as $\delta$ decreases. We average over a uniformly random sign vector and show that success forces large average correlation with this vector, while the sparse coordinate observations bound this correlation by a testing argument. We show that no learning rule with
\[
m\le \frac{R^2d}{512\epsilon^2}
\]
succeeds with probability at least $3/4$ under every distribution in $\mathfrak{D}^{\mathrm{dim}}_{R}$. If no positive integer satisfies this bound, there is nothing to prove. Otherwise, fix such an integer $m$, set $\rho:=4\epsilon/R\le 1/2$, and let $A_m:\mathcal{Z}^m\to B_R^{(\infty)}$ be any learning rule. For $\b\in\{\pm 1\}^d$, let $F_{\b}$ denote the population objective corresponding to $D^{\mathrm{dim}}_{\b,\rho}$. Then
\begin{align*}
F_{\b}(\x)&= -\sum_{j=1}^d\sum_{k\in\{\pm1\}} D^{\mathrm{dim}}_{\b,\rho}(j,k) kx_j \\
&= -\sum_{j=1}^d \frac{1}{d}x_j\left(1\cdot \frac{1+\rho b_j}{2}+(-1)\cdot \frac{1-\rho b_j}{2}\right) \\
&= -\frac{1}{d}\sum_{j=1}^d \rho b_jx_j \\
&= -\frac{\rho}{d}\langle \b, \x \rangle \qquad \forall \x\in B_R^{(\infty)}.
\end{align*}
Thus, $F_{\b}$ is minimized over $B_R^{(\infty)}$ at $\x_{D^{\mathrm{dim}}_{\b,\rho}}^\star=R\b$, and
\[
\min_{\x\in B_R^{(\infty)}}F_{\b}(\x)=-\rho R.
\]
Therefore,
\[
F_{\b}(\x)-\min_{\x\in B_R^{(\infty)}}F_{\b}(\x)=\rho R-\frac{\rho}{d}\langle \b,\x\rangle \qquad \forall \x\in B_R^{(\infty)}.
\]
Let $B$ be uniform on $\{\pm 1\}^d$. Conditional on $B=\b$, let
\[
S=((J_1,K_1),\dots,(J_m,K_m))\sim (D^{\mathrm{dim}}_{\b,\rho})^m,
\]
and define the normalized output vector
\[
Y:=\frac{A_m(S)}{R}\in[-1,1]^d.
\]
Write $Y_j$ for the $j$th coordinate of $Y$. Then
\[
F_{B}(A_m(S))-\min_{\x\in B_R^{(\infty)}}F_{B}(\x)=\rho R\left(1-\frac{1}{d}\langle B,Y\rangle\right).
\]
Fix $j\in[d]$, and define
\[
I_j:=\left\{i\in[m]:\ J_i=j\right\}, \qquad N_j:=|I_j|.
\]
Let
\[
T_j:=\left((B_\ell)_{\ell\ne j},\ (J_i)_{i\in[m]},\ (K_i)_{i\notin I_j}\right).
\]
Condition on $T_j$. Then $I_j$ and $N_j$ are fixed, the only remaining randomness lies in $B_j$ and $(K_i)_{i\in I_j}$, and $B_j$ remains uniform on $\{\pm 1\}$.
Moreover,
\[
(K_i)_{i\in I_j}\mid (T_j,B_j=1)\sim P_+^{\otimes N_j}, \qquad (K_i)_{i\in I_j}\mid (T_j,B_j=-1)\sim P_-^{\otimes N_j},
\]
where
\[
P_+(s)=\frac{1+\rho s}{2}, \qquad P_-(s)=\frac{1-\rho s}{2} \qquad \forall s\in\{\pm 1\}.
\]
Using total variation and Pinsker's inequality,
\begin{align*}
\E[B_jY_j\mid T_j]
&= \frac{1}{2}\left(\E[Y_j\mid T_j,B_j=1]-\E[Y_j\mid T_j,B_j=-1]\right) \\
&= \frac{1}{2}\left(\E_{P_+^{\otimes N_j}}[Y_j]-\E_{P_-^{\otimes N_j}}[Y_j]\right) \\
&\le \frac{1}{2}\left|\E_{P_+^{\otimes N_j}}[Y_j]-\E_{P_-^{\otimes N_j}}[Y_j]\right| \\
&\le \mathrm{TV}\left(P_+^{\otimes N_j},P_-^{\otimes N_j}\right) \\
&\le \sqrt{\frac{1}{2}\mathrm{KL}\left(P_+^{\otimes N_j}\|P_-^{\otimes N_j}\right)} \\
&\le \rho\sqrt{2N_j},
\end{align*}
where the last inequality uses product additivity of KL divergence and \Cref{lem:bernoulli-kl}. Taking total expectations and using Jensen's inequality gives
\[
\E[B_jY_j]= \E\left[\E[B_jY_j\mid T_j]\right]\le \rho\E[\sqrt{2N_j}]\le \rho\sqrt{2\E[N_j]}= \rho\sqrt{\frac{2m}{d}},
\]
where the last equality uses $\E[N_j]=\sum_{i=1}^m\Prob(J_i=j)=m/d$. Summing over $j$ gives
\[
\E\left[\frac{1}{d}\langle B,Y\rangle\right]=\frac{1}{d}\sum_{j=1}^d\E[B_jY_j]\le \rho\sqrt{\frac{2m}{d}}\le \frac{1}{4},
\]
where the last inequality follows from $m\le R^2d/(512\epsilon^2)$ and $\rho=4\epsilon/R$. On the other hand, if $A_m$ succeeded with probability at least $3/4$ under every distribution in $\mathfrak{D}^{\mathrm{dim}}_{R}$, then
\[
\Prob\left(F_{B}(A_m(S))-\min_{\x\in B_R^{(\infty)}}F_{B}(\x)\le \epsilon\right)\ge \frac{3}{4}.
\]
On this event,
\[
\frac{1}{d}\langle B,Y\rangle \ge 1-\frac{\epsilon}{\rho R}=\frac{3}{4},
\]
since $\rho=4\epsilon/R$. Since $\frac{1}{d}\langle B,Y\rangle\in[-1,1]$, this gives
\[
\E\left[\frac{1}{d}\langle B,Y\rangle\right]\ge \frac{3}{4}\cdot \frac{3}{4}+\frac{1}{4}\cdot (-1)=\frac{5}{16},
\]
contradicting the bound $\E[d^{-1}\langle B,Y\rangle]\le 1/4$ established above. Thus
\[
m_{B_R^{(\infty)}}^\star(\epsilon,1/4;\mathfrak{D}^{\mathrm{dim}}_{R})\ge \frac{R^2d}{512\epsilon^2}.
\]
By monotonicity in the confidence parameter, the same dimension lower bound holds for every $0<\delta\le1/4$.

\smallskip
\noindent\emph{Confidence term.}
We next isolate the one-coordinate coin test that yields the confidence term.
Set $\rho:=2\epsilon/R\le 1/4$. Let $A_m:\mathcal{Z}^m\to B_R^{(\infty)}$ be any learning rule. If $F_b$ denotes the population objective corresponding to $D^{\mathrm{conf}}_{b,\rho}$, then
\[
F_b(\x)=-\rho bx_1 \qquad \forall \x\in B_R^{(\infty)},
\]
so
\[
\min_{\x\in B_R^{(\infty)}}F_b(\x)=-\rho R.
\]
Define
\[
\widehat b(S):=\begin{cases}
+1, & [A_m(S)]_1>0,\\
-1, & [A_m(S)]_1\le 0.
\end{cases}
\]
If
\[
F_b(A_m(S))-\min_{\x\in B_R^{(\infty)}}F_b(\x)\le \epsilon,
\]
then
\[
\rho\left(R-b[A_m(S)]_1\right)\le \epsilon,
\]
which implies $b[A_m(S)]_1\ge R/2>0$, and hence $\widehat b(S)=b$. Therefore, if $A_m$ succeeded with probability at least $1-\delta$ under both $D^{\mathrm{conf}}_{+1,\rho}$ and $D^{\mathrm{conf}}_{-1,\rho}$, then for
\[
E:=\{\widehat b(S)=+1\},
\]
we would have
\[
P_+^{\otimes m}(E)\ge 1-\delta, \qquad P_-^{\otimes m}(E)\le \delta,
\]
where $P_+$ and $P_-$ denote the single-sample distributions $D^{\mathrm{conf}}_{+1,\rho}$ and $D^{\mathrm{conf}}_{-1,\rho}$. Since \Cref{lem:bernoulli-kl} gives
\[
\mathrm{KL}(P_+\|P_-)\le 4\rho^2,
\]
additivity of KL divergence under products and \Cref{lem:testing-inequalities} imply
\[
4m\rho^2\ge m\,\mathrm{KL}(P_+\|P_-)
=\mathrm{KL}(P_+^{\otimes m}\|P_-^{\otimes m})
\ge \frac12\log\frac{1}{2\delta}.
\]
Therefore,
\[
m\ge \frac{1}{8\rho^2}\log\left(\frac{1}{2\delta}\right)=\frac{R^2}{32\epsilon^2}\log\left(\frac{1}{2\delta}\right),
\]
where the equality uses $\rho=2\epsilon/R$. Hence
\[
m_{B_R^{(\infty)}}^\star(\epsilon,\delta;\mathfrak{D}^{\mathrm{conf}}_{R})\ge \frac{R^2}{32\epsilon^2}\log\frac{1}{2\delta}.
\]

By the inclusions $\mathfrak{D}^{\mathrm{dim}}_{R}\subseteq \mathfrak{D}^{\mathrm{lin}}_{R}$ and $\mathfrak{D}^{\mathrm{conf}}_{R}\subseteq \mathfrak{D}^{\mathrm{lin}}_{R}$,
\[
m_{B_R^{(\infty)}}^\star(\epsilon,\delta;\mathfrak{D}^{\mathrm{lin}}_{R})\ge \max\left\{\frac{R^2d}{512\epsilon^2},
\frac{R^2}{32\epsilon^2}\log\frac{1}{2\delta}\right\}.
\]
Since $\delta\le 1/4$, we have $\log(1/(2\delta))\ge \frac{1}{2}\log(1/\delta)$. Therefore,
\[
m_{B_R^{(\infty)}}^\star(\epsilon,\delta;\mathfrak{D}^{\mathrm{lin}}_{R})\ge \max\left\{\frac{R^2d}{512\epsilon^2},
\frac{R^2}{64\epsilon^2}\log\frac{1}{\delta}\right\}\ge \frac{R^2}{1024\epsilon^2}\left(d+\log\frac{1}{\delta}\right),
\]
where the last inequality uses $\max\{u,v\}\ge (u+v)/2$.

\section{Proofs for 1-Lipschitz stochastic optimization over \texorpdfstring{$X_R^{(2)}$}{X two}}
\label[appendix]{app:l2-proofs}

\subsection{Packing and tent lower bounds}

\begin{lemma}\label{lem:sparse-sign-packing}
There exists an absolute constant $c>0$ such that, for every $d\ge1$ and every integer $1\le s\le d$, one can find a set $\U\subseteq\{-1,0,1\}^d$ with $|\U|\ge2$ such that
\[
\|\u\|_2^2=s\quad \forall \u\in\U,\quad
\langle \u,\u'\rangle\le s/2\quad \forall \u,\u'\in\U,\ \u\ne\u',\quad
\text{and } \log|\U|\ge c s\log\frac{ed}{s}.
\]
\end{lemma}

\begin{proof}[Proof of \Cref{lem:sparse-sign-packing}]
Sample a random vector $\u\in\{-1,0,1\}^d$ by choosing $S=\{i\in[d]:u_i\ne0\}$ uniformly among all subsets of $[d]$ of size $s$, and then choosing independent signs on $S$. Let $\u'$ be an independent copy, set $S'=\{i\in[d]:u'_i\ne0\}$, and let $L:=|S\cap S'|$. We bound
\[
\Prob\left[\langle \u,\u'\rangle>\frac{s}{2}\right].
\]
First suppose that $d\le 4es$. Conditional on $S,S'$, the inner product is a sum of $L\le s$ independent signs. If $L=0$, the probability above is zero. Otherwise Hoeffding's inequality gives
\[
\Prob\left[\langle \u,\u'\rangle>\frac{s}{2}\right]\le \exp(-c_1s)\le \exp\left(-c_2s\log\frac{ed}{s}\right),
\]
for absolute constants $c_1,c_2>0$.

It remains to consider the case $d>4es$. The event $\langle \u,\u'\rangle>s/2$ implies $L>s/2$. Let $r_0:=\lceil s/2\rceil$. Conditioning on $S$ and union bounding over all $r_0$-subsets of $S$,
\[
\Prob[L\ge r_0]\le \binom{s}{r_0}\frac{\binom{d-r_0}{s-r_0}}{\binom{d}{s}}\le \binom{s}{r_0}\left(\frac{s}{d}\right)^{r_0}\le \left(\frac{es}{r_0}\cdot\frac{s}{d}\right)^{r_0}.
\]
Since $d>4es$, we have $2es/d<1/2$, and hence
\[
\Prob\left[L>\frac{s}{2}\right]\le \left(\frac{2es}{d}\right)^{s/2}\le \exp\left(-c_3s\log\frac{ed}{s}\right)
\]
for an absolute constant $c_3>0$. Set $c:=\frac18\min\{c_2,c_3\}$. The two cases give
\[
\Prob\left[\langle \u,\u'\rangle>\frac{s}{2}\right]\le \exp\left(-8c s\log\frac{ed}{s}\right).
\]

Let
\[
M:=\max\left\{2,\left\lfloor\exp\left(2c s\log\frac{ed}{s}\right)\right\rfloor\right\}.
\]
If $\u_1,\ldots,\u_M$ are sampled independently as above, then
\[
\Prob\left( \bigcup_{1\le i<j\le M} \left\{\langle \u_i,\u_j\rangle>\frac{s}{2}\right\} \right) \le \binom{M}{2}\exp\left(-8c s\log\frac{ed}{s}\right).
\]
If $M=2$, the right-hand side is less than one. If $M\ge3$, then $M\le \exp(2c s\log(ed/s))$, and therefore
\[
\binom{M}{2}\exp\left(-8c s\log\frac{ed}{s}\right)\le \frac12\exp\left(-4c s\log\frac{ed}{s}\right)<1.
\]
Thus there is a realization with all pairwise inner products at most $s/2$. Let $\U$ be the resulting set of $M$ vectors. It remains to check its size. If $\exp(c s\log(ed/s))\le2$, then $M\ge\exp(c s\log(ed/s))$. Otherwise,
\[
\left\lfloor\exp\left(2c s\log\frac{ed}{s}\right)\right\rfloor\ge \exp\left(2c s\log\frac{ed}{s}\right)-1\ge \exp\left(c s\log\frac{ed}{s}\right).
\]
This proves the lemma.
\end{proof}

\begin{lemma}\label{lem:l2-integer-packing}
There exists an absolute constant $c>0$ such that, for every $d\ge 1$ and $R\ge 1$, one can find a set $\W\subseteq X_R^{(2)}$ with $|\W|\ge2$ and a radius $r\in[R/2,R]$ such that
\[
\|\w\|_2=r\quad \forall \w\in\W,\quad
\langle \w,\w'\rangle\le r^2/2\quad \forall \w,\w'\in\W,\ \w\ne\w',\quad
\text{and } \log |\W|\ge c H_2(d,R).
\]
In particular, $\|\w-\w'\|_2^2\ge r^2$ for distinct $\w,\w'\in\W$.
\end{lemma}

\begin{proof}[Proof of \Cref{lem:l2-integer-packing}]
Set $s:=\min\{d,\lfloor R^2\rfloor\}$ and apply \Cref{lem:sparse-sign-packing}. If $s=\lfloor R^2\rfloor$, set $\W:=\U$ and $r:=\sqrt{s}$. Then $\W\subseteq X_R^{(2)}$, and since $R\ge1$, $\frac{R}{2}\le r\le R$. Otherwise $s=d<\lfloor R^2\rfloor$. Set
\[
q:=\left\lfloor \frac{R}{\sqrt d}\right\rfloor,\qquad \W:=\{q\u:\ \u\in\U\},\qquad r:=q\sqrt d.
\]
Since $d=s<\lfloor R^2\rfloor$, we have $R/\sqrt d>1$, and hence $r\le R$. If $R/\sqrt d<2$, then $q=1$ and $r=\sqrt d>R/2$. If $R/\sqrt d\ge2$, then $q\ge R/(2\sqrt d)$, so again $r\ge R/2$. Thus $r\in[R/2,R]$. Also $\W\subseteq X_R^{(2)}$, since $q\u\in\Z^d$ and $\|q\u\|_2=r\le R$ for every $\u\in\U$. In both cases, every $\w\in\W$ has norm $r$, and for distinct $\w,\w'\in\W$,
\[
\langle \w,\w'\rangle \le \frac{r^2}{2}.
\]
Finally
\[
\log|\W|=\log|\U|\ge c s\log\frac{ed}{s}=c H_2(d,R).
\]
This proves the lemma.
\end{proof}

\begin{lemma}\label{lem:tent-lower-from-packing}
Let $R>0$ and let $X\subseteq B_R^{(2)}$. Suppose that there are $r>0$ and a finite set $\W\subseteq X$ with $|\W|\ge2$ such that
\[
\|\w\|_2=r\quad \forall \w\in\W,\qquad
\|\w-\w'\|_2\ge r\quad \forall \w,\w'\in\W,\ \w\ne\w'.
\]
Then there are a finite sample space $\mathcal{Z}$, an anchored $1$-Lipschitz loss map $f:B_R^{(2)}\times\mathcal{Z}\to\R$, and a family $\mathfrak{D}$ of probability distributions on $\mathcal{Z}$ such that, for every $0<\epsilon\le r/16$ and every $0<\delta\le1/4$,
\[
m_X^\star(\epsilon,\delta;\mathfrak{D})\ge c\frac{r^2}{\epsilon^2}\left(\log|\W|+\log\frac{1}{\delta}\right),
\]
where $c>0$ is an absolute constant.
\end{lemma}

\begin{proof}[Proof of \Cref{lem:tent-lower-from-packing}]
For $\w\in\W$, define
\[
\psi_{\w}(\x):=\begin{cases}
\frac{r}{4}-\|\x-\w\|_2, & \text{if }\|\x-\w\|_2<\frac{r}{4},\\
0, & \text{if }\|\x-\w\|_2\ge \frac{r}{4},
\end{cases}
\qquad \x\in B_R^{(2)}.
\]
This is the tent of height $r/4$ centered at $\w$. Each $\psi_{\w}$ is $1$-Lipschitz, and $\psi_{\w}(\0)=0$ because $\|\w\|_2=r>r/4$. If $\psi_{\w}(\x)>0$ and $\psi_{\w'}(\x)>0$ for distinct $\w,\w'\in\W$, then $\|\w-\w'\|_2<r/2$, contradicting the separation assumption. Thus the positive supports of the tents are disjoint.

Take $\mathcal{Z}:=2^\W$ as the sample space. For $V\subseteq\W$, define
\[
f(\x;V):=-\max_{\w\in V}\psi_{\w}(\x),
\]
with the convention that the maximum over the empty set is zero. The maximum of $1$-Lipschitz functions is $1$-Lipschitz, and multiplying by $-1$ preserves the Lipschitz constant. Since every tent vanishes at the origin, $f(\0;V)=0$ for every $V\subseteq\W$.

For $\u\in\W$ and $0<\rho\le1/2$, let $D_{\u,\rho}$ be the distribution on $\mathcal{Z}$ under which the membership indicators $(\1_{\{\w\in V\}})_{\w\in\W}$ are independent and
\[
\Prob_{V\sim D_{\u,\rho}}[\u\in V]=\frac{1+\rho}{2},\qquad \Prob_{V\sim D_{\u,\rho}}[\w\in V]=\frac{1-\rho}{2} \quad \forall \w\in\W\setminus\{\u\}.
\]
Let
\[
\mathfrak{D}:=\{D_{\u,\rho}:\u\in\W,\ 0<\rho\le1/2\}.
\]

Fix $\u\in\W$ and write $F_{\u,\rho}$ for the population objective under $D_{\u,\rho}$. Since at most one tent is positive at any point,
\[
F_{\u,\rho}(\x)=-\frac{1+\rho}{2}\psi_{\u}(\x)-\frac{1-\rho}{2}\sum_{\w\in\W\setminus\{\u\}}\psi_{\w}(\x).
\]
Hence $F_{\u,\rho}(\u)=-(1+\rho)r/8$. This is the unique population minimum over $X$. Indeed, inside the positive support of $\psi_{\u}$,
\[
F_{\u,\rho}(\x)=-\frac{1+\rho}{2}\psi_{\u}(\x)\ge -\frac{1+\rho}{8}r,
\]
with equality only at $\x=\u$. Outside this support, either $\x$ lies in another tent and $F_{\u,\rho}(\x)\ge-(1-\rho)r/8$, or no tent is active and $F_{\u,\rho}(\x)=0$. Therefore
\[
F_{\u,\rho}(\x)-F_{\u,\rho}(\u)\ge \frac{\rho r}{4}
\]
whenever $\psi_{\u}(\x)=0$.

Assume $0<\epsilon\le r/16$ and set $\rho:=8\epsilon/r$. Then $\rho\le1/2$, and hence $D_{\u,\rho}\in\mathfrak{D}$. Also $\rho r/4=2\epsilon$. Since the positive supports are disjoint, any $\epsilon$-optimal output identifies the hidden point:
\[
\x\in X,\quad
F_{\u,\rho}(\x)-\min_{\z\in X}F_{\u,\rho}(\z)\le\epsilon
\quad\Longrightarrow\quad
\x \text{ lies in the positive support of } \psi_{\u}.
\]

We now prove the testing claim. Suppose that a decoder $\widehat{\u}:\mathcal{Z}^m\to\W$ satisfies
\[
\Prob_{S\sim D_{\u,\rho}^{\otimes m}}[\widehat{\u}(S)=\u]\ge1-\delta\qquad \forall \u\in\W
\]
for some $0<\delta\le1/4$. Then
\[
m\ge \frac{1}{96}\rho^{-2}\left(\log|\W|+\log\frac{1}{\delta}\right).
\]
We first prove the confidence term. Fix distinct $\u,\v\in\W$ and let
\[
E:=\{S\in\mathcal{Z}^m:\widehat{\u}(S)=\u\}.
\]
The success guarantee gives $D_{\u,\rho}^{\otimes m}(E)\ge1-\delta$. Since $E\subseteq\{\widehat{\u}(S)\ne\v\}$, it also gives $D_{\v,\rho}^{\otimes m}(E)\le\delta$. Thus
\[
D_{\u,\rho}^{\otimes m}(E)\ge1-\delta,\qquad D_{\v,\rho}^{\otimes m}(E)\le\delta.
\]
By \Cref{lem:testing-inequalities},
\[
m\,\mathrm{KL}(D_{\u,\rho}\|D_{\v,\rho})=\mathrm{KL}(D_{\u,\rho}^{\otimes m}\|D_{\v,\rho}^{\otimes m})\ge \frac12\log\frac{1}{2\delta},
\]
where the equality is additivity under products. The distributions $D_{\u,\rho}$ and $D_{\v,\rho}$ differ only in coordinates $\u$ and $\v$, so \Cref{lem:bernoulli-kl} gives $\mathrm{KL}(D_{\u,\rho}\|D_{\v,\rho})\le8\rho^2$. Since $\delta\le1/4$,
\[
m\ge \frac{1}{32}\rho^{-2}\log\frac{1}{\delta}.
\]

If $\log|\W|<4\log2$, then, since $\delta\le1/4$,
\[
\log|\W|+\log\frac{1}{\delta}\le3\log\frac{1}{\delta}.
\]
Hence the confidence bound gives
\[
m\ge \frac{1}{96}\rho^{-2}\left(\log|\W|+\log\frac{1}{\delta}\right).
\]
It remains to prove the testing claim when $\log|\W|\ge4\log2$. Let $P_0$ be the reference distribution on $\mathcal{Z}$ under which the membership indicators are independent and
\[
\Prob_{V\sim P_0}[\w\in V]=\frac{1-\rho}{2} \qquad \forall \w\in\W.
\]
By \Cref{lem:testing-inequalities}, followed by additivity of KL divergence under products and \Cref{lem:bernoulli-kl},
\[
\begin{aligned}
\frac1{|\W|}\sum_{\u\in\W}\Prob_{S\sim D_{\u,\rho}^{\otimes m}}[\widehat{\u}(S)\ne \u]
&\ge 1-\frac{|\W|^{-1}\sum_{\u\in\W}\mathrm{KL}(D_{\u,\rho}^{\otimes m}\|P_0^{\otimes m})+\log2}{\log|\W|} \\
&= 1-\frac{m|\W|^{-1}\sum_{\u\in\W}\mathrm{KL}(D_{\u,\rho}\|P_0)+\log2}{\log|\W|} \\
&\ge 1-\frac{4m\rho^2+\log2}{\log|\W|}.
\end{aligned}
\]
Here the last inequality uses the fact that $D_{\u,\rho}$ and $P_0$ differ only in coordinate $\u$.
Since the decoder has error probability at most $\delta$ for every $\u$,
\[
\delta\ge \frac1{|\W|}\sum_{\u\in\W}\Prob_{S\sim D_{\u,\rho}^{\otimes m}}[\widehat{\u}(S)\ne \u]
\ge 1-\frac{4m\rho^2+\log2}{\log|\W|}.
\]
As $\delta\le1/4$, this gives
\[
4m\rho^2+\log2\ge \frac34\log|\W|.
\]
Since $\log|\W|\ge4\log2$, we have
\[
m\ge \frac18\rho^{-2}\log|\W|.
\]
Together with the small-entropy case above, this proves the testing claim.

Now let $A_m:\mathcal{Z}^m\to X$ be any proper algorithm that is $\epsilon$-optimal with probability at least $1-\delta$ for every distribution in $\mathfrak{D}$. Define $\widehat{\u}(S)$ to be the unique $\w$ whose positive support contains $A_m(S)$, and define it arbitrarily if no such $\w$ exists. The optimization-to-identification implication above gives
\[
\Prob_{S\sim D_{\u,\rho}^{\otimes m}}[\widehat{\u}(S)=\u]\ge1-\delta \qquad \forall \u\in\W.
\]
The testing claim therefore gives
\[
m\ge \frac{1}{96\rho^{2}}\left(\log|\W|+\log\frac{1}{\delta}\right).
\]
Since $\rho=8\epsilon/r$, $1/\rho^{2}={r^2}/(64\epsilon^2)$, this proves the lemma.
\end{proof}

\subsection{Proof of \Cref{thm:l2-continuous-nonconvex-upper-lower}}

By \Cref{lem:volumetric-cover}, for every $0<\eta<1$,
\[
N(\eta R,B_R^{(2)},\|\cdot\|_2)\le \left(1+\frac{2}{\eta}\right)^d\le \left(\frac{3}{\eta}\right)^d.
\]
\Cref{lem:anchored-lipschitz-uc-vdw}, applied with $K=B_R^{(2)}$, $\|\cdot\|=\|\cdot\|_2$, and $h=d$, gives the stated upper bound.

It remains to prove the lower bound. Apply \Cref{lem:sparse-sign-packing} with $s=d$, and let $c_1>0$ be its absolute constant. This gives a set $\U\subseteq\{\pm1\}^d$ with $|\U|\ge2$, $\|\u\|_2^2=d$ for every $\u\in\U$, $\langle \u,\u'\rangle\le d/2$ for distinct $\u,\u'\in\U$, and $\log|\U|\ge c_1d$. Set
\[
\W:=\left\{\frac{R}{\sqrt d}\u:\ \u\in\U\right\}.
\]
Then $\W\subseteq B_R^{(2)}$, $|\W|\ge2$, and $\|\w\|_2=R$ for every $\w\in\W$. Also, if distinct $\w,\w'\in\W$ are obtained from $\u,\u'\in\U$, then
\[
\langle \w,\w'\rangle=\frac{R^2}{d}\langle \u,\u'\rangle\le \frac{R^2}{2},\qquad
\|\w-\w'\|_2^2=2R^2-2\langle \w,\w'\rangle\ge R^2,
\]
and $\log|\W|=\log|\U|\ge c_1d$. Apply \Cref{lem:tent-lower-from-packing} with $X=B_R^{(2)}$ and $r=R$. If $c_0\le1/16$, then $0<\epsilon\le c_0R$ implies $0<\epsilon\le R/16$. Therefore,
\[
m_{B_R^{(2)}}^\star(\epsilon,\delta;\mathfrak{D})\ge c_2\frac{R^2}{\epsilon^2}\left(\log|\W|+\log\frac{1}{\delta}\right)
\]
for an absolute constant $c_2>0$. Since $\log|\W|\ge c_1d$, this gives the lower bound with constant $c_2\min\{c_1,1\}$.

\subsection{Proof of \Cref{thm:l2-uc-upper}}

If $R<1$, then $X_R^{(2)}=\{\0\}$ and anchored uniform convergence is trivial. We therefore assume $R\ge1$.

First suppose that $\lfloor R^2\rfloor\ge d$. Then $H_2(d,R)=d$. Since $X_R^{(2)}\subseteq B_R^{(2)}$, the upper bound in \Cref{thm:l2-continuous-nonconvex-upper-lower} gives
\[
m_{X_R^{(2)}}^{\mathrm{a-UC}}(\epsilon,\delta;\mathfrak{D})
\le \left\lceil C\frac{R^2}{\epsilon^2}\left(d+\log\frac{1}{\delta}\right)\right\rceil
= \left\lceil C\frac{R^2}{\epsilon^2}\left(H_2(d,R)+\log\frac{1}{\delta}\right)\right\rceil.
\]

It remains to consider the case $1\le \lfloor R^2\rfloor<d$. Every point of $X_R^{(2)}$ has support size at most $\lfloor R^2\rfloor$. For a fixed support size $k$, the support can be chosen in $\binom{d}{k}$ ways and the signs in at most $2^k$ ways. Given the support and signs, the squared magnitudes are positive integers with total at most $\lfloor R^2\rfloor$, giving at most $\binom{\lfloor R^2\rfloor}{k}$ choices. Hence
\[
|X_R^{(2)}|\le 1+\sum_{k=1}^{\lfloor R^2\rfloor} \binom{d}{k}2^k\binom{\lfloor R^2\rfloor}{k}.
\]
For $1\le k\le \lfloor R^2\rfloor$, using $\binom{d}{k}\le (ed/k)^k$ and $\binom{\lfloor R^2\rfloor}{k}\le (e\lfloor R^2\rfloor/k)^k$ gives
\[
\binom{d}{k}2^k\binom{\lfloor R^2\rfloor}{k}
\le \left(\frac{2e^2d\lfloor R^2\rfloor}{k^2}\right)^k.
\]
The logarithm of the right-hand side is at most $C_0\lfloor R^2\rfloor\log(ed/\lfloor R^2\rfloor)$ for an absolute constant $C_0>0$. Since $H_2(d,R)=\lfloor R^2\rfloor\log(ed/\lfloor R^2\rfloor)$ in this case, we have
\[
|X_R^{(2)}|\le 1+\lfloor R^2\rfloor\exp(C_0H_2(d,R)).
\]
Since $\log(\lfloor R^2\rfloor+1)\le \lfloor R^2\rfloor\le H_2(d,R)$, this gives
\[
\log |X_R^{(2)}|\le (C_0+1)H_2(d,R).
\]
Since $H_2(d,R)\ge1$, for every $0<\eta<1$,
\[
N(\eta R,X_R^{(2)},\|\cdot\|_2)\le |X_R^{(2)}|
\le \left(\frac{3}{\eta}\right)^{(C_0+1)H_2(d,R)}.
\]
Let $C_1$ be the absolute constant in \Cref{lem:anchored-lipschitz-uc-vdw}. Applying that lemma with $K=X_R^{(2)}$, $\|\cdot\|=\|\cdot\|_2$, and $h=(C_0+1)H_2(d,R)$ gives the desired bound in this case with constant $(C_0+1)C_1$. The two cases prove the upper bound with constant $\max\{C,(C_0+1)C_1\}$.

It remains to prove the lower bound. 
Let $\W\subseteq X_R^{(2)}$ and $r\in[R/2,R]$ be given by \Cref{lem:l2-integer-packing}. Apply \Cref{lem:tent-lower-from-packing} with $X=X_R^{(2)}$. If $c_0\le1/32$, then $0<\epsilon\le c_0R$ implies $0<\epsilon\le r/16$. Therefore,
\[
m_{X_R^{(2)}}^\star(\epsilon,\delta;\mathfrak{D})\ge c_1\frac{r^2}{\epsilon^2}\left(\log|\W|+\log\frac{1}{\delta}\right),
\]
where $c_1>0$ is the constant in \Cref{lem:tent-lower-from-packing}. If $\log|\W|\ge c_2H_2(d,R)$ is the bound from \Cref{lem:l2-integer-packing}, then $r\ge R/2$ gives the claimed lower bound with constant $c_1\min\{c_2,1\}/4$.

\section{Proofs for \texorpdfstring{$\mu$-strongly convex and $L$-smooth losses}{mu-strongly convex and L-smooth losses}}
\label[appendix]{app:sc-proofs}

\subsection{Auxiliary localization and empirical process bounds}

\begin{lemma}\label{lem:sc-localization}
Let $R\in[1,\infty]$, let $F:\R^d\to\R$ be $\mu$-strongly convex and $L$-smooth, and let $\x^\star\in\argmin_{\x\in X_R^{(\infty)}}F(\x)$. Then, for every $\x\in X_R^{(\infty)}$ and $s\ge0$,
\[
F(\x)-F(\x^\star)\le s \quad\Longrightarrow\quad \|\x-\x^\star\|_2\le 2\sqrt{d\min\{\kappa,\lfloor R\rfloor^2\}+\frac{s}{\mu}}.
\]
\end{lemma}

\begin{proof}[Proof of \Cref{lem:sc-localization}]
Let $\u\in\argmin_{\y\in B_{\lfloor R\rfloor}^{(\infty)}}F(\y)$. When $R<\infty$, the feasible integer set is exactly $B_{\lfloor R\rfloor}^{(\infty)}\cap\Z^d$. When $R=\infty$, $\u$ is the unconstrained minimizer. We first construct $\q\in X_R^{(\infty)}$ such that $\|\q-\u\|_2\le \sqrt d/2$ and $\langle\nabla F(\u),\q-\u\rangle=0$. If $R=\infty$, let $\q$ be a coordinatewise nearest-integer rounding of $\u$; then $\nabla F(\u)=0$.

If $R<\infty$, define $\q$ coordinatewise as follows. For every coordinate with $u_j\in(-\lfloor R\rfloor,\lfloor R\rfloor)$, let $q_j$ be a nearest integer to $u_j$. For every coordinate with $u_j=\pm\lfloor R\rfloor$, set $q_j=u_j$. Since $\lfloor R\rfloor$ is an integer, this gives $\q\in X_R^{(\infty)}$ and $\|\q-\u\|_2\le\sqrt d/2$. First-order optimality over the box gives $\nabla_jF(\u)=0$ on every interior coordinate, while $q_j-u_j=0$ on every boundary coordinate. Thus $\langle\nabla F(\u),\q-\u\rangle=0$.

The upper inequality in \eqref{eq:sc-first-order} gives
\[
F(\q)-F(\u)\le \langle\nabla F(\u),\q-\u\rangle+\frac{L}{2}\|\q-\u\|_2^2\le \frac{Ld}{8}.
\]
By optimality of $\x^\star$ over $X_R^{(\infty)}$,
\[
F(\x^\star)-F(\u)\le \frac{Ld}{8}.
\]
First-order optimality of $\u$ over $B_{\lfloor R\rfloor}^{(\infty)}$ gives $\langle\nabla F(\u),\y-\u\rangle\ge0$ for every $\y\in B_{\lfloor R\rfloor}^{(\infty)}$. Hence the lower inequality in \eqref{eq:sc-first-order} gives
\[
F(\y)-F(\u)\ge \frac{\mu}{2}\|\y-\u\|_2^2\qquad \forall \y\in B_{\lfloor R\rfloor}^{(\infty)}.
\]
Applying this bound to $\x^\star$ gives
\[
\|\x^\star-\u\|_2\le \frac12\sqrt{\kappa d}.
\]
If $F(\x)-F(\x^\star)\le s$, then
\[
F(\x)-F(\u)\le F(\x^\star)-F(\u)+s\le \frac{Ld}{8}+s,
\]
and another application of strong convexity gives
\[
\|\x-\u\|_2\le \sqrt{\frac{\kappa d}{4}+\frac{2s}{\mu}}.
\]
The triangle inequality gives
\[
\|\x-\x^\star\|_2\le \frac12\sqrt{\kappa d}+\sqrt{\frac{\kappa d}{4}+\frac{2s}{\mu}}
\le 2\sqrt{\kappa d+\frac{s}{\mu}},
\]
and, when $R<\infty$, the diameter bound gives $\|\x-\x^\star\|_2\le2\lfloor R\rfloor\sqrt d$. Combining these two bounds yields
\[
\|\x-\x^\star\|_2\le 2\sqrt{d\min\{\kappa,\lfloor R\rfloor^2\}+\frac{s}{\mu}}.
\]
For $R=\infty$, this is the preceding bound. This proves the lemma.
\end{proof}

We next isolate the empirical-process input used in the ERM proof: a high-probability chaining bound for sub-Gaussian increments, followed by its Euclidean-ball specialization.

\begin{lemma}\label{lem:subgaussian-chaining}
Let $(T,\rho)$ be a nonempty compact metric space. Let $(G_t)_{t\in T}$ be a stochastic process such that, for every $s,t\in T$ and every $\lambda\in\R$,
\[
\E\exp\left(\lambda(G_s-G_t)\right)\le \exp\left(\frac{\lambda^2\rho(s,t)^2}{2}\right).
\]
Set $\Delta:=\sup_{s,t\in T}\rho(s,t)$. Then there is an absolute constant $C>0$ such that, for every $u\ge0$, with probability at least $1-2e^{-u^2}$,
\[
\sup_{s,t\in T}|G_s-G_t|
\le C\left(\int_0^\Delta\sqrt{\log N(\varepsilon,T,\rho)}\,d\varepsilon+u\Delta\right).
\]
\end{lemma}

\begin{proof}[Proof of \Cref{lem:subgaussian-chaining}]
For a real random variable $X$, write
\[
\|X\|_{\psi_2}:=\inf\left\{K>0:\ \E\exp\left(\frac{X^2}{K^2}\right)\le2\right\}.
\]
Fix $s,t\in T$. If $s=t$, then $G_s-G_t=0$. If $s\ne t$, put $X=(G_s-G_t)/\rho(s,t)$. For every $\alpha\in\R$, applying the assumed exponential bound with $\lambda=\alpha/\rho(s,t)$ gives
\[
\E\exp(\alpha X)
=\E\exp\left(\frac{\alpha}{\rho(s,t)}(G_s-G_t)\right)
\le \exp\left(\frac{\alpha^2}{2}\right).
\]
By Jensen's inequality, $\alpha\E X\le \alpha^2/2$ for every $\alpha\in\R$, and hence $\E X=0$. Thus \citet[Proposition~2.5.2]{vershynin2018high}, applied to $X$, gives an absolute constant $C_1>0$ such that $\|X\|_{\psi_2}\le C_1$. Therefore,
\[
\|G_s-G_t\|_{\psi_2}=\rho(s,t)\|X\|_{\psi_2}\le C_1\rho(s,t).
\]
Thus the process $(G_t)_{t\in T}$ satisfies \citet[Definition~8.1.1]{vershynin2018high}. By \citet[Theorem~8.1.6]{vershynin2018high}, there is an absolute constant $C_2>0$ such that, with probability at least $1-2e^{-u^2}$,
\[
\sup_{s,t\in T}|G_s-G_t|
\le C_1C_2\left(\int_0^\infty \sqrt{\log N(\varepsilon,T,\rho)}\,d\varepsilon+u\Delta\right).
\]
Since $N(\varepsilon,T,\rho)=1$ for every $\varepsilon\ge \Delta$, the integral over $[\Delta,\infty)$ is zero. Taking $C:=C_1C_2$ proves the claim.
\end{proof}

\begin{lemma}\label{lem:subgaussian-euclidean-ball}
Let $A\subseteq B_2(\x_0,R)\subseteq\R^d$ be nonempty and compact, and assume that $\x_0\in A$. Let $a>0$, and let $(G_\x)_{\x\in A}$ be a stochastic process with $G_{\x_0}=0$ such that, for every $\x,\y\in A$ and every $\lambda\in\R$,
\[
\E\exp\left(\lambda(G_\x-G_\y)\right)\le \exp\left(\frac{\lambda^2a^2\|\x-\y\|_2^2}{2}\right).
\]
Then there is an absolute constant $C>0$ such that, for every $t>0$, with probability at least $1-e^{-t}$,
\[
\sup_{\x\in A}G_\x\le CaR(\sqrt d+\sqrt t).
\]
\end{lemma}

\begin{proof}[Proof of \Cref{lem:subgaussian-euclidean-ball}]
The case $R=0$ is immediate because $A=\{\x_0\}$. Now assume $R>0$.

Apply \Cref{lem:subgaussian-chaining} to the compact metric space $A$ with metric $d_A(\x,\y):=a\|\x-\y\|_2$. Since $A\subseteq B_2(\x_0,R)$, the quantity $\Delta:=\sup_{\x,\y\in A}d_A(\x,\y)$ satisfies $\Delta\le2aR$. Also, applying \Cref{lem:volumetric-cover} to $A-\x_0$ with the norm $a\|\cdot\|_2$ gives
\[
N(\varepsilon,A,d_A)\le \left(1+\frac{2aR}{\varepsilon}\right)^d,\qquad 0<\varepsilon\le 2aR.
\]
Therefore,
\[
\int_0^{2aR}\sqrt{\log N(\varepsilon,A,d_A)}\,d\varepsilon
\le C_1aR\sqrt d
\]
for an absolute constant $C_1>0$. Indeed, after the change of variables $\varepsilon=2aRu$,
\[
\int_0^{2aR}\sqrt{d\log\left(1+\frac{2aR}{\varepsilon}\right)}\,d\varepsilon
=2aR\sqrt d\int_0^1\sqrt{\log(1+1/u)}\,du,
\]
and $\int_0^1\sqrt{\log(1+1/u)}\,du<+\infty$.
Because $G_{\x_0}=0$, $\sup_{\x\in A}G_\x\le \sup_{\x,\y\in A}|G_\x-G_\y|$. Taking $u=\sqrt{t+\log 2}$ in \Cref{lem:subgaussian-chaining} gives an absolute constant $C_2>0$ such that, with probability at least $1-e^{-t}$,
\[
\sup_{\x\in A}G_\x\le C_2\left(C_1aR\sqrt d+2aR\sqrt{t+\log2}\right).
\]
With $C:=C_2(C_1+2\sqrt{\log2}+2)$, the last bound is at most $CaR(\sqrt d+\sqrt t)$, which proves the lemma.
\end{proof}

\subsection{Proof of \Cref{thm:sc-anchored-uc}}

We first prove the upper bound. Fix $1\le R<\infty$ and let $D\in\mathfrak{D}^{\mathrm{sc}}_{\mu,L,\sigma}$. Let $S=(z_1,\ldots,z_m)\sim D^m$ and define the anchored empirical process
\[
G_\x:=\bar F_D(\x)-\bar F_S(\x),\qquad \x\in X_R^{(\infty)}.
\]
By the anchored normalization, $G_{\0}=0$. For any $\x,\y\in X_R^{(\infty)}$,
\[
G_\x-G_\y=-\frac1m\sum_{i=1}^m\left(f(\x;z_i)-f(\y;z_i)-F_D(\x)+F_D(\y)\right).
\]
Independence and the increment condition give, for every $\lambda\in\R$,
\[
\E\exp\left(\lambda(G_\x-G_\y)\right)
\le \exp\left(\frac{\lambda^2\sigma^2\|\x-\y\|_2^2}{2m}\right).
\]
Moreover, $X_R^{(\infty)}=\{-\lfloor R\rfloor,\ldots,\lfloor R\rfloor\}^d\subseteq B_2(\0,\lfloor R\rfloor\sqrt d)$. Applying \Cref{lem:subgaussian-euclidean-ball} with $A=X_R^{(\infty)}$, $\x_0=\0$, radius $\lfloor R\rfloor\sqrt d$, and $a=\sigma/\sqrt m$, we get that with probability at least $1-e^{-t}$,
\[
\sup_{\x\in X_R^{(\infty)}}G_\x\le C_1\frac{\sigma\lfloor R\rfloor\sqrt d}{\sqrt m}(\sqrt d+\sqrt t),
\]
where $C_1>0$ is an absolute constant. The same estimate applies to $-G$. Taking $t=\log(2/\delta)$ and union bounding gives, with probability at least $1-\delta$,
\[
\sup_{\x\in X_R^{(\infty)}}|G_\x|
\le C_1\frac{\sigma\lfloor R\rfloor\sqrt d}{\sqrt m}\left(\sqrt d+\sqrt{\log\frac{2}{\delta}}\right).
\]
The right-hand side is at most $\epsilon$ whenever
\[
m\ge 9C_1^2\frac{\sigma^2\lfloor R\rfloor^2d}{\epsilon^2}\left(d+\log\frac{1}{\delta}\right),
\]
since $\log(2/\delta)\le2(d+\log(1/\delta))$. This proves the upper bound.

It remains to prove the lower bound.
Let $\mathcal{Z}:=\R^d$, let $D_0:=N(\0,\sigma^2I_d)$ be the distribution of $Z$ on $\mathcal{Z}$, and consider the loss map
\[
f(\x;Z):=\frac{\mu}{2}\|\x\|_2^2-\langle Z,\x\rangle,\qquad \x\in\R^d,\ Z\in\R^d.
\]
Then $f(\0;Z)=0$ for every $Z$. Each sample loss is $\mu$-strongly convex and $\mu$-smooth, and hence $L$-smooth. For this distribution,
\[
f(\x;Z)-f(\y;Z)-F_{D_0}(\x)+F_{D_0}(\y)=-\langle Z,\x-\y\rangle,
\]
which is $\sigma\|\x-\y\|_2$-sub-Gaussian. Hence $D_0$ belongs to the admissible distribution class associated with this loss map.

Fix $1\le R<\infty$. Since the loss is already anchored, $\bar F_{D_0}=F_{D_0}$ and $\bar F_S=F_S$. If $S=(Z_1,\ldots,Z_m)\sim D_0^m$ and $\bar Z:=\frac{1}{m}\sum_{i=1}^m Z_i$, then
\[
F_{D_0}(\x)-F_S(\x)=\langle \bar Z,\x\rangle.
\]
Since $X_R^{(\infty)}=\{-\lfloor R\rfloor,\ldots,\lfloor R\rfloor\}^d$,
\[
\sup_{\x\in X_R^{(\infty)}}|F_{D_0}(\x)-F_S(\x)|=\lfloor R\rfloor\|\bar Z\|_1.
\]
Writing $\bar Z=(\sigma/\sqrt m)\xi$ with $\xi\sim N(\0,I_d)$, the anchored uniform deviation is $(\sigma\lfloor R\rfloor/\sqrt m)\|\xi\|_1$.

\smallskip
\noindent\emph{Dimension term.}
Since $\xi_1$ is standard normal, choose an absolute constant $\eta_0>0$ such that $\Prob[|\xi_1|<\eta_0]\le1/8$. Then $\E|\{j:|\xi_j|<\eta_0\}|\le d/8$, so Markov's inequality gives
\[
\Prob\left[|\{j:|\xi_j|<\eta_0\}|\ge \frac d2\right]\le\frac14.
\]
Hence, with probability at least $3/4$, $\|\xi\|_1\ge \eta_0d/2$.
Consequently, if $m<\eta_0^2\sigma^2\lfloor R\rfloor^2d^2/(4\epsilon^2)$, then the anchored uniform deviation is larger than $\epsilon$ with probability at least $3/4$. Since $\delta\le1/4$, any sample size satisfying anchored UC must also satisfy
\[
m\ge c_1\frac{\sigma^2\lfloor R\rfloor^2d^2}{\epsilon^2}
\]
for an absolute constant $c_1>0$.

\smallskip
\noindent\emph{Confidence term.}
Since $\sum_{j=1}^d \xi_j\sim N(0,d)$, the Gaussian tail lower bound gives an absolute constant $\eta_1>0$ such that, for every $0<\delta\le1/4$,
\[
\Prob\left[\left|\sum_{j=1}^d \xi_j\right|\ge \eta_1\sqrt{d\log\frac{1}{\delta}}\right]\ge 2\delta.
\]
Since $\|\xi\|_1\ge|\sum_{j=1}^d\xi_j|$, the same lower bound holds with $\|\xi\|_1$ in place of $|\sum_{j=1}^d\xi_j|$.
Therefore, if $m<\eta_1^2\sigma^2\lfloor R\rfloor^2d\log(1/\delta)/\epsilon^2$, then the anchored uniform deviation is larger than $\epsilon$ with probability greater than $\delta$. Thus every valid anchored UC sample size also satisfies
\[
m\ge c_2\frac{\sigma^2\lfloor R\rfloor^2d}{\epsilon^2}\log\frac{1}{\delta}
\]
for an absolute constant $c_2>0$. Combining the two lower bounds and using $\max\{u,v\}\ge (u+v)/2$ proves
\[
m_{X_R^{(\infty)}}^{\mathrm{a-UC}}(\epsilon,\delta;\mathfrak{D}^{\mathrm{sc}}_{\mu,L,\sigma})
\ge c\frac{\sigma^2\lfloor R\rfloor^2d}{\epsilon^2}\left(d+\log\frac{1}{\delta}\right)
\]
for an absolute constant $c>0$.

\subsection{Proof of the upper bound in \Cref{thm:sc-erm-upper-lower}}

Fix $R\in[1,\infty]$, $D\in\mathfrak{D}^{\mathrm{sc}}_{\mu,L,\sigma}$, and a sample $S\sim D^m$. Let
\[
\x^\star\in\argmin_{\x\in X_R^{(\infty)}}F_D(\x),\qquad \hat \x_S\in\argmin_{\x\in X_R^{(\infty)}}F_S(\x)
\]
be the selected population and empirical minimizers. Define
\[
\Delta_D(\x):=F_D(\x)-F_D(\x^\star),\qquad
\Delta_S(\x):=F_S(\x)-F_S(\x^\star),
\]
and set $G(\x):=\Delta_D(\x)-\Delta_S(\x)$. Since $F_S(\hat \x_S)\le F_S(\x^\star)$, we have $\Delta_S(\hat \x_S)\le0$, and hence $\Delta_D(\hat \x_S)\le G(\hat \x_S)$.

For any $\x,\y\in X_R^{(\infty)}$ and any $\lambda\in\R$, the increment assumption and independence give
\[
\E\exp\left(\lambda(G(\x)-G(\y))\right)\le \exp\left(\frac{\lambda^2\sigma^2\|\x-\y\|_2^2}{2m}\right).
\]
For $s\ge0$, set $A_s:=\{\x\in X_R^{(\infty)}:\Delta_D(\x)\le s\}$. By \Cref{lem:sc-localization},
\[
A_s\subseteq B_2\left(\x^\star,2\sqrt{d\min\{\kappa,\lfloor R\rfloor^2\}+\frac{s}{\mu}}\right).
\]
The set $A_s$ is finite, hence compact. Indeed, this is immediate when $R<\infty$, while for $R=\infty$ the localization bound makes $A_s$ a bounded subset of the lattice. Since $G(\x^\star)=0$, applying \Cref{lem:subgaussian-euclidean-ball} with $\x_0=\x^\star$, radius $2\sqrt{d\min\{\kappa,\lfloor R\rfloor^2\}+s/\mu}$, and $a=\sigma/\sqrt m$ gives, with probability at least $1-e^{-t}$,
\[
\sup_{\x\in A_s}G(\x)\le C_1\sigma\sqrt{\frac{(d\min\{\kappa,\lfloor R\rfloor^2\}+s/\mu)(d+t)}{m}},
\]
where $C_1>0$ is an absolute constant.

We prove the relative bound $G(\x)\le \Delta_D(\x)/2+r$ uniformly over $X_R^{(\infty)}$ by peeling the level sets of $\Delta_D$. Let $T:=d+\log(1/\delta)$ and
\[
r:=C_2\left(\sigma\sqrt{\frac{d\min\{\kappa,\lfloor R\rfloor^2\}T}{m}}+\frac{\sigma^2T}{\mu m}\right),
\]
where $C_2>0$ is fixed below. Apply this bound with $s_k=2^kr$ and $t_k=\log(1/\delta)+(k+1)^2$, so that $e^{-t_k}=\delta e^{-(k+1)^2}$. For $k\ge0$, let $E_k$ be the event that
\[
\sup_{\x\in A_{2^kr}}G(\x)> C_1\sigma\sqrt{\frac{(d\min\{\kappa,\lfloor R\rfloor^2\}+2^kr/\mu)(T+(k+1)^2)}{m}}.
\]
The estimate gives $\Prob(E_k)\le e^{-t_k}$, and hence
\[
\Prob\left(\bigcup_{k\ge0}E_k\right)\le \sum_{k\ge0}e^{-t_k}=\delta\sum_{k\ge0}e^{-(k+1)^2}<\delta.
\]
Let $\mathcal E:=\bigcap_{k\ge0}E_k^c$. Then $\Prob(\mathcal E)\ge1-\delta$, and on $\mathcal E$, for every $k\ge0$,
\[
\sup_{\x\in A_{2^kr}}G(\x)\le C_1\sigma\sqrt{\frac{(d\min\{\kappa,\lfloor R\rfloor^2\}+2^kr/\mu)(T+(k+1)^2)}{m}}.
\]
Since $T\ge1$ and $(k+1)^2\le 4\cdot 2^kT$, the right-hand side is bounded as
\[
\begin{aligned}
\sup_{\x\in A_{2^kr}}G(\x)
&\le 3C_1\sigma\sqrt{\frac{(d\min\{\kappa,\lfloor R\rfloor^2\}+2^kr/\mu)2^kT}{m}} \\
&\le 3C_1\,2^{k/2}\sigma\sqrt{\frac{d\min\{\kappa,\lfloor R\rfloor^2\}T}{m}}+3C_1\,2^k\sigma\sqrt{\frac{rT}{\mu m}} \\
&\le 3C_1\,2^{k/2}\frac{r}{C_2}+3C_1\,2^k\frac{r}{\sqrt{C_2}},
\end{aligned}
\]
where the last line uses the definition of $r$. Choose $C_2\ge1$ so that $3C_1/C_2+3C_1/\sqrt{C_2}\le1/4$. Then
\[
\sup_{\x\in A_{2^kr}}G(\x)\le \begin{cases}
r, & k=0,\\
2^{k-2}r, & k\ge1.
\end{cases}
\]
Therefore, on $\mathcal E$,
\[
G(\x)\le \frac12\Delta_D(\x)+r\qquad \forall \x\in X_R^{(\infty)}.
\]
Indeed, if $\Delta_D(\x)\le r$, then $G(\x)\le r$. If $2^{k-1}r<\Delta_D(\x)\le2^kr$ for some $k\ge1$, then $\x\in A_{2^kr}$ and $G(\x)\le 2^{k-2}r<\Delta_D(\x)/2$.

Applying the relative bound to $\hat \x_S$ gives
\[
\Delta_D(\hat \x_S)\le G(\hat \x_S)\le \frac12\Delta_D(\hat \x_S)+r,
\]
and hence $\Delta_D(\hat \x_S)\le2r$. Thus, with probability at least $1-\delta$,
\[
F_D(\hat \x_S)-F_D(\x^\star)
\le 2C_2\left(\sigma\sqrt{\frac{d\min\{\kappa,\lfloor R\rfloor^2\}(d+\log(1/\delta))}{m}}
+\frac{\sigma^2}{\mu}\frac{d+\log(1/\delta)}{m}\right).
\]
Consequently, every integer
\[
m\ge \left\lceil 16C_2^2\sigma^2\left(d+\log\frac{1}{\delta}\right)
\max\left\{\frac{d\min\{\kappa,\lfloor R\rfloor^2\}}{\epsilon^2},\frac{1}{\mu\epsilon}\right\}\right\rceil
\]
makes the excess bound at most $\epsilon$. Since $\kappa\ge1$ and $\lfloor R\rfloor\ge1$, the assumption $\epsilon\le\mu d$ gives
\[
\frac{1}{\mu\epsilon}\le \frac{d\min\{\kappa,\lfloor R\rfloor^2\}}{\epsilon^2}.
\]
This proves the desired upper bound for ERM sample complexity.

\subsection{Proof of the lower bound in \Cref{thm:sc-erm-upper-lower}}

\begin{lemma}\label{lem:sc-block-gadget}
Assume that $\kappa=L/\mu\ge64$, and let $\tau$ be a positive integer satisfying $\tau^2\le\kappa/16$.
For $b\in\{\pm1\}$ and $\gamma>0$, define
\[
\phi_b(x,y):=\mu(x-\tau y)^2+\frac{L}{4}\left(y-\frac{1}{2}\right)^2-\gamma b\frac{x}{\tau},\qquad (x,y)\in\R^2.
\]
Assume $0<\gamma\le\mu/24$. For each $b\in\{\pm1\}$, the function $\phi_b$ is $\mu$-strongly convex and $L$-smooth on $\R^2$. Moreover, $\phi_{+1}$ has the unique minimizer $(\tau,1)$ over $\Z^2$, while $\phi_{-1}$ has the unique minimizer $(0,0)$ over $\Z^2$. Finally, if
\[
\widehat b_{\mathrm{block}}(x,y):=\begin{cases}
+1, & (x,y)=(\tau,1),\\
-1, & (x,y)\ne(\tau,1),
\end{cases}
\]
then, for every $b\in\{\pm1\}$ and every $(x,y)\in\Z^2$,
\[
\phi_b(x,y)-\min_{(x',y')\in\Z^2}\phi_b(x',y')\ge \frac{\gamma}{24}\left(1-b\,\widehat b_{\mathrm{block}}(x,y)\right).
\]
\end{lemma}

\begin{proof}[Proof of \Cref{lem:sc-block-gadget}]
The Hessian of the quadratic part is
\[
\begin{pmatrix}
2\mu & -2\mu \tau\\
-2\mu \tau & 2\mu \tau^2+L/2
\end{pmatrix}.
\]
Its determinant is $\mu L$. Also, using $\tau^2\le \kappa/16$, its trace is at most
\begin{equation}\label{eq:sc-block-trace-bound}
2\mu+2\mu \tau^2+\frac{L}{2}\le 2\mu+\frac{L}{8}+\frac{L}{2}\le L,
\end{equation}
where the last inequality uses $\kappa\ge64$. Therefore the largest eigenvalue is at most $L$, and the smallest eigenvalue is at least $\mu L/L=\mu$. Thus every block is $\mu$-strongly convex and $L$-smooth.

The quadratic part ties exactly the two integer points $(0,0)$ and $(\tau,1)$. Indeed, for an integer point $(x,y)\in\Z^2$, write $p:=x-\tau y$. Then $p\in\Z$, and after subtracting the value at $(0,0)$,
\[
\mu(x-\tau y)^2+\frac{L}{4}\left(y-\frac{1}{2}\right)^2-\frac{L}{16}=\mu p^2+\frac{L}{4}y(y-1).
\]
This quantity is zero exactly at $(0,0)$ and $(\tau,1)$. At every other integer point,
\[
\mu p^2+\frac{L}{4}y(y-1)\ge \frac{\mu}{4}\left(p^2+\kappa y(y-1)\right)\ge \frac{\mu}{4}.
\]
The linear perturbation is controlled by the same integer gap. We claim that
\[
\left|\frac{x}{\tau}\right|+\left|\frac{x}{\tau}-1\right|\le 3\left(p^2+\kappa y(y-1)\right)\qquad \forall (x,y)\in\Z^2\setminus\{(0,0),(\tau,1)\}.
\]
Indeed, since $x/\tau=y+p/\tau$, if $y\in\{0,1\}$, then the excluded endpoints force $p\ne0$. As $p\in\Z$ and $\tau\ge1$,
\[
\left|\frac{x}{\tau}\right|+\left|\frac{x}{\tau}-1\right|
\le 1+2\frac{|p|}{\tau}
\le 3p^2.
\]
If $y\notin\{0,1\}$, then $y(y-1)\ge1$ and $|y|+|y-1|\le3y(y-1)$ for integer $y$. Thus
\[
\left|\frac{x}{\tau}\right|+\left|\frac{x}{\tau}-1\right|
\le |y|+|y-1|+2\frac{|p|}{\tau}
\le 3y(y-1)+2p^2
\le 3\left(p^2+\kappa y(y-1)\right),
\]
where the last inequality uses $\kappa\ge1$.

Assume $0<\gamma\le \mu/24$. Then, at every integer point outside $\{(0,0),(\tau,1)\}$,
\[
\gamma\left|\frac{x}{\tau}-1\right|+\gamma\left|\frac{x}{\tau}\right|
\le \frac{\mu}{8}\left(p^2+\kappa y(y-1)\right).
\]
For $b=+1$,
\[
\phi_{+1}(x,y)-\phi_{+1}(\tau,1)=\mu p^2+\frac{L}{4}y(y-1)-\gamma\left(\frac{x}{\tau}-1\right).
\]
This equals $\gamma$ at $(0,0)$. At every integer point outside $\{(0,0),(\tau,1)\}$, the preceding two inequalities give a lower bound of
\[
\frac{\mu}{8}\left(p^2+\kappa y(y-1)\right)\ge 3\gamma.
\]
Thus $(\tau,1)$ is the unique integer minimizer of $\phi_{+1}$. Similarly,
\[
\phi_{-1}(x,y)-\phi_{-1}(0,0)=\mu p^2+\frac{L}{4}y(y-1)+\gamma\frac{x}{\tau}.
\]
This equals $\gamma$ at $(\tau,1)$, and the same bound shows that it is at least $3\gamma$ at every other integer point outside $\{(0,0),(\tau,1)\}$. Thus $(0,0)$ is the unique integer minimizer of $\phi_{-1}$.
The right-hand side in the block error bound is zero when the decoder is correct and at most $\gamma/12$ when it is wrong. Since the estimates above show that every wrongly decoded point has excess at least $\gamma$, the claimed bound follows.
\end{proof}

\begin{proof}[Proof of the lower bound in \Cref{thm:sc-erm-upper-lower}]
We first prove the result when $\kappa\ge64$. Let $c_1>0$ be an absolute constant such that $1-24c_1\ge7/8$. Assume $0<\epsilon\le c_1\mu d/72$. Let $\bar d=\lfloor d/2\rfloor$ be the number of two-dimensional blocks, and set
\[
\tau=\max\left\{1,\left\lfloor\frac{\min\{\sqrt{\kappa},\lfloor R\rfloor\}}{4}\right\rfloor\right\},\qquad
\gamma:=\frac{\epsilon}{c_1\bar d}.
\]
Then $1\le\tau\le\lfloor R\rfloor$, $\tau^2\le\kappa/16$, and
\[
\tau^2\ge \frac{1}{64}\min\{\kappa,\lfloor R\rfloor^2\}.
\]
Indeed, if $\min\{\sqrt{\kappa},\lfloor R\rfloor\}<8$, then $\tau=1$, otherwise $\tau\ge\min\{\sqrt{\kappa},\lfloor R\rfloor\}/8$.
Since $\bar d\ge d/3$ for $d\ge2$, this choice gives $\gamma\le\mu/24$. Let $\phi_b$ and $\widehat b_{\mathrm{block}}$ be the block function and block decoder from \Cref{lem:sc-block-gadget}. On the sample space $\R^d$, define
\[
f(\x;Z):=\sum_{j=1}^{\bar d}\left[\mu(x_{2j-1}-\tau x_{2j})^2+\frac{L}{4}\left(x_{2j}-\frac{1}{2}\right)^2\right]+\frac{\mu}{2}\sum_{\ell=2\bar d+1}^d x_\ell^2-\langle Z,\x\rangle.
\]
The Hessian calculation in \Cref{lem:sc-block-gadget} shows that $f(\cdot\,;Z)$ is $\mu$-strongly convex and $L$-smooth for every $Z\in\R^d$.

Let $\mathbf e_1,\ldots,\mathbf e_d$ be the standard basis of $\R^d$. For $\b\in\{\pm1\}^{\bar d}$, set
\[
\bm{\theta}_{\b}:=\sum_{j=1}^{\bar d}\frac{\gamma b_j}{\tau}\mathbf e_{2j-1},
\]
and let $D_{\b}:=N(\bm{\theta}_{\b},\sigma^2I_d)$. The corresponding population objective is
\[
F_{\b}(\x):=\E_{Z\sim D_{\b}}f(\x;Z)=
\begin{cases}
\displaystyle\sum_{j=1}^{\bar d} \phi_{b_j}(x_{2j-1},x_{2j}), & d \text{ is even},\\[0.25em]
\displaystyle\sum_{j=1}^{\bar d} \phi_{b_j}(x_{2j-1},x_{2j})+\frac{\mu}{2}x_d^2, & d \text{ is odd},
\end{cases}
\qquad \x\in\R^d.
\]
Define the block decoder
\[
\widehat b_j(\x):=\widehat b_{\mathrm{block}}(x_{2j-1},x_{2j}),\qquad j=1,\ldots,\bar d.
\]
Write $\widehat{\b}(\x):=(\widehat b_1(\x),\ldots,\widehat b_{\bar d}(\x))$. The block minimizers $(0,0)$ and $(\tau,1)$ belong to the box because $\tau\le\lfloor R\rfloor$ and $R\ge1$, and the extra coordinate is minimized at $0$ when $d$ is odd. Thus the minimum of $F_{\b}$ over $X_R^{(\infty)}$ is the sum of the corresponding block minima. Summing the block error estimate in \Cref{lem:sc-block-gadget} over the $\bar d$ blocks gives
\[
F_{\b}(\x)-\min_{\z\in X_R^{(\infty)}}F_{\b}(\z)\ge \frac{\gamma}{24}\left(\bar d-\langle \widehat{\b}(\x),\b\rangle\right)\qquad \forall \x\in X_R^{(\infty)}.
\]

Moreover,
\[
f(\x;Z)-f(\y;Z)-F_{\b}(\x)+F_{\b}(\y)=-\langle Z-\bm{\theta}_{\b},\x-\y\rangle,
\]
which is $\sigma\|\x-\y\|_2$-sub-Gaussian. Hence these distributions belong to the admissible class for this loss map. If $\b,\b'\in\{\pm1\}^{\bar d}$, then
\[
\|\bm{\theta}_{\b}-\bm{\theta}_{\b'}\|_2^2=\frac{\gamma^2}{\tau^2}\|\b-\b'\|_2^2\le 64\frac{\gamma^2}{\min\{\kappa,\lfloor R\rfloor^2\}}\|\b-\b'\|_2^2.
\]

\smallskip
\noindent\emph{Dimension term.}
Suppose that a learning rule $A_m:(\R^d)^m\to X_R^{(\infty)}$ is $\epsilon$-optimal with probability at least $3/4$ under every $D_{\b}$.
Let $B$ be uniform on $\{\pm1\}^{\bar d}$. Conditional on $B=\b$, let
\[
S=(Z_1,\ldots,Z_m)\sim D_{\b}^m,
\]
and define the decoded output vector
\[
Y:=\widehat{\b}(A_m(S))\in\{\pm1\}^{\bar d}.
\]
On the success event,
\[
\frac{1}{\bar d}\langle B,Y\rangle\ge 1-\frac{24\epsilon}{\gamma\bar d}=1-24c_1\ge\frac{7}{8}.
\]
Since $\bar d^{-1}\langle B,Y\rangle\in[-1,1]$, this implies
\[
\E\left[\frac{1}{\bar d}\langle B,Y\rangle\right]\ge \frac{3}{4}\cdot \frac{7}{8}+\frac{1}{4}\cdot(-1)=\frac{13}{32}.
\]

To upper bound the same correlation, fix $j\in[\bar d]$ and let
\[
T_j:=\left((B_s)_{s\ne j},\ (Z_{i,\ell})_{i\in[m],\,\ell\ne 2j-1}\right).
\]
Condition on $T_j$. The only remaining randomness lies in $B_j$ and $(Z_{i,2j-1})_{i\in[m]}$, and $B_j$ remains uniform on $\{\pm1\}$. Moreover,
\[
(Z_{i,2j-1})_{i\in[m]}\mid (T_j,B_j=1)\sim Q_+^{\otimes m}, \qquad (Z_{i,2j-1})_{i\in[m]}\mid (T_j,B_j=-1)\sim Q_-^{\otimes m},
\]
where
\[
Q_+:=N(\gamma/\tau,\sigma^2),\qquad Q_-:=N(-\gamma/\tau,\sigma^2).
\]
Using total variation and Pinsker's inequality,
\begin{align*}
\E[B_jY_j\mid T_j]
&= \frac{1}{2}\left(\E[Y_j\mid T_j,B_j=1]-\E[Y_j\mid T_j,B_j=-1]\right) \\
&= \frac{1}{2}\left(\E_{Q_+^{\otimes m}}[Y_j]-\E_{Q_-^{\otimes m}}[Y_j]\right) \\
&\le \frac{1}{2}\left|\E_{Q_+^{\otimes m}}[Y_j]-\E_{Q_-^{\otimes m}}[Y_j]\right| \\
&\le \mathrm{TV}\left(Q_+^{\otimes m},Q_-^{\otimes m}\right) \\
&\le \sqrt{\frac{1}{2}\mathrm{KL}\left(Q_+^{\otimes m}\|Q_-^{\otimes m}\right)} \\
&= \frac{\gamma\sqrt m}{\sigma \tau}\le 8\sqrt{\frac{m\gamma^2}{\sigma^2\min\{\kappa,\lfloor R\rfloor^2\}}},
\end{align*}
where the equality uses the Gaussian identity in \Cref{lem:elementary-kl-bounds}, and the last inequality uses the lower bound on $\tau^2$. Taking expectations and summing over $j$ yields
\[
\E\left[\frac{1}{\bar d}\langle B,Y\rangle\right]\le 8\sqrt{\frac{m\gamma^2}{\sigma^2\min\{\kappa,\lfloor R\rfloor^2\}}}.
\]
Comparing this upper bound with the lower bound $13/32$, and using $\gamma=\epsilon/(c_1\bar d)$ and $\bar d\ge d/3$, shows that there is an absolute constant $c_2>0$ such that
\[
m\ge c_2\frac{\sigma^2\min\{\kappa,\lfloor R\rfloor^2\}d^2}{\epsilon^2}.
\]
By monotonicity in the confidence parameter, this dimension lower bound holds for every $0<\delta\le1/4$.

\smallskip
\noindent\emph{Confidence term.}
Take the two hypotheses $\b^+:=\1_{\bar d}$ and $\b^-:=-\1_{\bar d}$, using the value of $\gamma$ fixed above. If a learning rule is $\epsilon$-optimal with probability at least $1-\delta$ under both hypotheses, then success under $D_{\b^+}$ gives $\langle \widehat{\b}(A_m(S)),\1_{\bar d}\rangle\ge 7\bar d/8$, while success under $D_{\b^-}$ gives $\langle \widehat{\b}(A_m(S)),\1_{\bar d}\rangle\le -7\bar d/8$. Thus the event
\[
E:=\{\langle \widehat{\b}(A_m(S)),\1_{\bar d}\rangle\ge0\}
\]
defines a test between $D_{\b^+}^{\otimes m}$ and $D_{\b^-}^{\otimes m}$ with error probability at most $\delta$. By \Cref{lem:testing-inequalities},
\[
\mathrm{KL}(D_{\b^+}^{\otimes m}\|D_{\b^-}^{\otimes m})\ge \frac12\log\frac{1}{2\delta}.
\]
On the other hand, the Gaussian identity in \Cref{lem:elementary-kl-bounds} gives
\[
\mathrm{KL}(D_{\b^+}^{\otimes m}\|D_{\b^-}^{\otimes m})=\frac{m}{2\sigma^2}\|\bm{\theta}_{\b^+}-\bm{\theta}_{\b^-}\|_2^2\le 128\frac{m\gamma^2\bar d}{\sigma^2\min\{\kappa,\lfloor R\rfloor^2\}}.
\]
Since $\delta\le1/4$, we have $\log(1/(2\delta))\ge \frac12\log(1/\delta)$. 

Combining the resulting confidence lower bound with the dimension lower bound gives, for an absolute constant $c_3>0$,
\[
m\ge c_3\frac{\sigma^2\min\{\kappa,\lfloor R\rfloor^2\}d}{\epsilon^2}\left(d+\log\frac{1}{\delta}\right),
\]
for all $\epsilon\le c_1\mu d/72$.

It remains to consider $1\le\kappa<64$. Assume again that $0<\epsilon\le c_1\mu d/72$, and set $\gamma:=\epsilon/(c_1d)$. On $\R^d$, take the sample loss
\[
f(\x;Z):=\frac{\mu}{2}\|\x\|_2^2-\langle Z,\x\rangle,
\]
and, for each $\b\in\{\pm1\}^d$, let $D_{\b}:=N((\mu/2)\1+\gamma\b,\sigma^2I_d)$. The Hessian is $\mu I_d$, so the sample losses are $\mu$-strongly convex and $\mu$-smooth. Since $L\ge\mu$, they are admissible for the given $(\mu,L)$ class. The corresponding population objective is
\[
F_{\b}(\x):=\E_{Z\sim D_{\b}}f(\x;Z)
=\frac{\mu}{2}\|\x\|_2^2-\left\langle \frac{\mu}{2}\1+\gamma\b,\x\right\rangle .
\]
The same Gaussian calculation verifies the sub-Gaussian increment condition.
Since $\epsilon\le c_1\mu d/72$, we have $\gamma\le\mu/72$. Therefore, the integer minimizer in coordinate $j$ is $1$ if $b_j=1$ and $0$ if $b_j=-1$. These minimizers belong to $X_R^{(\infty)}$ because $R\ge1$. Define $\widehat b_j(\x)=1$ when $x_j\ge1$ and $\widehat b_j(\x)=-1$ when $x_j\le0$, and write $\widehat{\b}(\x):=(\widehat b_1(\x),\ldots,\widehat b_d(\x))$. Then every wrongly decoded coordinate costs at least $\gamma$, so
\[
F_{\b}(\x)-\min_{\z\in X_R^{(\infty)}}F_{\b}(\z)\ge \frac{\gamma}{2}\left(d-\langle\widehat{\b}(\x),\b\rangle\right)\qquad \forall \x\in X_R^{(\infty)}.
\]
The preceding dimension and confidence testing arguments apply with $\bar d$ replaced by $d$. The only change from the two-dimensional block construction is that the one-coordinate Gaussian mean separation is $2\gamma$ rather than $2\gamma/\tau$, so the testing estimates have no $\kappa$ loss. Therefore, for an absolute constant $c_4>0$,
\[
m_{X_R^{(\infty)}}^\star(\epsilon,\delta;\mathfrak{D}^{\mathrm{sc}}_{\mu,L,\sigma})\ge c_4\frac{\sigma^2}{\epsilon^2}\left(d^2+d\log\frac{1}{\delta}\right).
\]
Since $R\ge1$ gives $\lfloor R\rfloor\ge1$, we have $\min\{\kappa,\lfloor R\rfloor^2\}\le\kappa<64$. Hence the last inequality implies the claimed bound with constant $c_4/64$. Taking $c_0:=c_1/72$ and $c:=\min\{c_3,c_4/64\}$ completes the proof.
\end{proof}

\subsection{Proof of \Cref{thm:sc-continuous-erm}}

The proof follows the proof of the upper bound in \Cref{thm:sc-erm-upper-lower}, with the lattice localization step replaced by its continuous analogue. Fix a nonempty closed convex set $K\subseteq\R^d$, a distribution $D\in\mathfrak{D}^{\mathrm{sc}}_{\mu,L,\sigma}$, and a sample $S\sim D^m$. Let
\[
\x_D^\star\in\argmin_{\x\in K}F_D(\x),\qquad \hat \x_S\in\argmin_{\x\in K}F_S(\x)
\]
be the population and empirical minimizers. They exist because $K$ is closed and the strongly convex objectives are coercive, and they are unique by strong convexity. Define $\Delta_D,\Delta_S$, and $G$ as in that proof, with $\x^\star$ replaced by $\x_D^\star$. Then $\Delta_D(\hat \x_S)\le G(\hat \x_S)$, and the same sub-Gaussian increment bound holds for $G$ with parameter $\sigma/\sqrt m$.

The localization step is sharper in the continuous problem and uses only strong convexity. Fix $\x\in K$ and $0<t<1$, and set $\x_t:=(1-t)\x_D^\star+t\x$. Since $K$ is convex, $\x_t\in K$. Since $\x_D^\star$ minimizes $F_D$ over $K$, strong convexity gives
\[
F_D(\x_D^\star)\le F_D(\x_t)\le (1-t)F_D(\x_D^\star)+tF_D(\x)-\frac{\mu}{2}t(1-t)\|\x-\x_D^\star\|_2^2.
\]
Dividing by $t$ and letting $t\downarrow0$ gives
\[
\Delta_D(\x)=F_D(\x)-F_D(\x_D^\star)\ge \frac{\mu}{2}\|\x-\x_D^\star\|_2^2\qquad \forall \x\in K.
\]
Thus, for $s\ge0$, the level set
\[
A_s=\{\x\in K:\Delta_D(\x)\le s\} \subseteq B_2\left(\x_D^\star,\sqrt{\frac{2s}{\mu}}\right).
\]
Unlike the lattice localization in \Cref{lem:sc-localization}, this bound has no $\kappa d$ term. The set $A_s$ is compact because it is closed and bounded by the inclusion above. Applying \Cref{lem:subgaussian-euclidean-ball} with $A=A_s$, $\x_0=\x_D^\star$, radius $\sqrt{2s/\mu}$, and $a=\sigma/\sqrt m$ gives an absolute constant $C_1>0$ such that, with probability at least $1-e^{-t}$,
\[
\sup_{\x\in A_s}G(\x)\le C_1\sigma\sqrt{\frac{s(d+t)}{\mu m}}.
\]
Let $T:=d+\log(1/\delta)$ and set
\[
r:=C_2\frac{\sigma^2T}{\mu m},
\]
where $C_2>0$ is fixed below. Apply the preceding bound with $s_k=2^kr$ and $t_k=\log(1/\delta)+(k+1)^2$. For $k\ge0$, let $E_k$ be the event that
\[
\sup_{\x\in A_{2^kr}}G(\x)> C_1\sigma\sqrt{\frac{2^kr(T+(k+1)^2)}{\mu m}}.
\]
The estimate gives $\Prob(E_k)\le e^{-t_k}$, and hence
\[
\Prob\left(\bigcup_{k\ge0}E_k\right)\le \sum_{k\ge0}e^{-t_k}=\delta\sum_{k\ge0}e^{-(k+1)^2}<\delta.
\]
Let $\mathcal E:=\bigcap_{k\ge0}E_k^c$. Then $\Prob(\mathcal E)\ge1-\delta$, and on $\mathcal E$, for every $k\ge0$,
\[
\sup_{\x\in A_{2^kr}}G(\x)\le C_1\sigma\sqrt{\frac{2^kr(T+(k+1)^2)}{\mu m}}.
\]
Since $T\ge1$ and $(k+1)^2\le4\cdot2^kT$, the right-hand side is at most
\[
C_3\,2^k\sigma\sqrt{\frac{rT}{\mu m}}
=C_3\,2^k\frac{r}{\sqrt{C_2}}
\]
for an absolute constant $C_3>0$. Choose $C_2$ so that $C_3/\sqrt{C_2}\le1/4$. Then
\[
\sup_{\x\in A_{2^kr}}G(\x)\le2^{k-2}r\qquad \forall k\ge0.
\]
Therefore, on $\mathcal E$,
\[
G(\x)\le \frac12\Delta_D(\x)+r\qquad \forall \x\in K.
\]
Indeed, if $\Delta_D(\x)\le r$, then $\x\in A_r$ and $G(\x)\le r$. If $2^{k-1}r<\Delta_D(\x)\le2^kr$ for some $k\ge1$, then $\x\in A_{2^kr}$ and $G(\x)\le2^{k-2}r<\Delta_D(\x)/2$.
Applying the relative bound to $\hat \x_S$ gives
\[
\Delta_D(\hat \x_S)\le G(\hat \x_S)\le \frac12\Delta_D(\hat \x_S)+r,
\]
and hence $\Delta_D(\hat \x_S)\le2r$. Therefore, with probability at least $1-\delta$,
\[
F_D(\hat \x_S)-F_D(\x_D^\star)
\le 2C_2\frac{\sigma^2}{\mu m}\left(d+\log\frac{1}{\delta}\right).
\]
The stated sample complexity bound follows by requiring the right-hand side to be at most $\epsilon$.

\end{document}